\documentclass[11pt]{article}

\usepackage[left=1in,top=1in,right=1in,bottom=1in,dvips,letterpaper]{geometry} 
\usepackage[leqno]{amsmath}
\usepackage{amsxtra, amsfonts,amscd,  amssymb, graphicx,subfigure}
\usepackage{hyperref}
\usepackage{xcolor}
\hypersetup{
     colorlinks   = true,
     citecolor    = blue,
     linkcolor 	  = red
}

\usepackage[algo2e, vlined,ruled,linesnumbered]{algorithm2e}
\numberwithin{equation}{section}

\def\proof{\par{\it Proof}. \ignorespaces}

\newtheorem{theorem}{Theorem}
\newtheorem{lemma}[theorem]{Lemma}

\newtheorem{remark}[theorem]{Remark}

\newtheorem{assumptions}{Assumptions}
\newtheorem{assumption}[assumptions]{Assumption}
\newcounter{subassumption}[assumptions]
\renewcommand{\thesubassumption}{(\textit{\roman{subassumption}})}
\makeatletter
\renewcommand{\p@subassumption}{\theassumption}
\makeatother
\newcommand{\assume}{
  \refstepcounter{subassumption}%
  \par\thesubassumption~\ignorespaces}

\DeclareMathOperator{\tr}{tr}

\DeclareMathOperator{\diag}{diag}
\DeclareMathOperator{\dom}      {dom}

\newcommand{\Acal}{{\cal A}}

\newcommand{\Ical}{{\cal I}}

\newcommand{\Pcal}{{\cal P}}

\newcommand{\Xcal}{{\cal X}}

\newcommand{\Rmbb}{\mathbb{R}}

\newcommand{\br}{\mathbb{R}}

\newcommand{\be}{\begin{equation}}
\newcommand{\ee}{\end{equation}}
\newcommand{\ba}{\begin{array}}
\newcommand{\ea}{\end{array}}
\newcommand{\bea}{\begin{eqnarray}}
\newcommand{\eea}{\end{eqnarray}}
\newcommand{\beaa}{\begin{eqnarray*}}
\newcommand{\eeaa}{\end{eqnarray*}}

\begin{document}

\title{Practical Inexact Proximal Quasi-Newton Method with Global  Complexity Analysis}

\author{Katya Scheinberg
\thanks{ \texttt {katyas@lehigh.edu}. Department of Industrial and Systems Engineering, Lehigh University,
Harold S. Mohler Laboratory, 200 West Packer Avenue, Bethlehem, PA 18015-1582, USA. The work of this author is partially supported by NSF Grants DMS 10-16571, DMS 13-19356, 
 AFOSR Grant FA9550-11-1-0239, and  DARPA grant FA 9550-12-1-0406 negotiated by AFOSR.} \\
\and
Xiaocheng Tang
\thanks{ \texttt{xct@lehigh.edu}, Department of Industrial and Systems Engineering, Lehigh University,
Harold S. Mohler Laboratory, 200 West Packer Avenue, Bethlehem, PA 18015-1582, USA. The work of this author is partially supported by   DARPA grant FA 9550-12-1-0406 negotiated by AFOSR}
}

\titlepage

%

\date{\today}

\maketitle

\begin{abstract}
Recently several methods were proposed for sparse optimization which make careful use of second-order information \cite{Hsieh2011,nGLMNET,Olsen2012,Chin2012} to improve local convergence  rates. These methods construct a composite quadratic  approximation using Hessian information, optimize this approximation using a first-order method, such as coordinate descent and employ a line search to  ensure sufficient descent.
Here we propose a general framework, which includes slightly modified versions of existing algorithms and also a new algorithm, which uses limited memory BFGS Hessian approximations,   and provide a novel global convergence rate analysis, which covers methods  that solve subproblems via  coordinate descent. 

\end{abstract}


\section{Introduction} 
\label{sec:introduction}
In this paper, we are interested in the following popular convex optimization problem: \bea\label{prob:P} \min_{x\in\br^n} F(x)\equiv
f(x)+g(x), \eea where $f,g:\br^n\rightarrow\br$ are both convex
functions such that  $\nabla f(x)$ is assumed to be Lipschitz continuous with Lipschitz constant $L(f)$, i.e., \beaa \|\nabla f(x) - \nabla f(y) \|_2 \leq L(f) \|x-y\|_2, \quad \forall x,y\in\br^n,\eeaa
and $g(x)$ is convex and has some structure that can be exploited. In particular,  in much of the research on first order methods for problem \eqref{prob:P}  $g(x)$ is considered to be such that   the following problem has a closed form solution
 for any $z\in\br^n$:
 \[\min_{x\in\br^n}\left\{g(x)+\frac{1}{2}\|x-z\|^2 \right\}.
 \]
 
 Here our general requirement on $g(x)$ is slightly different - we assume that the following problem is
 computationally inexpensive to solve approximately, relative to minimizing $F(x)$
 for any $z\in\br^n$ and some class of positive definite matrices $H$:
 \bea\label{prob:g-shrink}\min_{x\in\br^n}\left\{g(x)+\frac{1}{2}\|x-z\|_H^2 \right\}.\eea
 Here $\|y\|^2_H$ denotes $y^\top H y$.   Clearly, the computational cost of approximately solving \eqref{prob:g-shrink} depends on the choice of matrix $H$ and the solution approach. 
 
We are particularly interested in the  case of sparse optimization, where $g(x)=\lambda \|x\|_1$.
While the theory we present here applies to the general form \eqref{prob:P},  the efficient method for solving \eqref{prob:g-shrink} that we consider in this paper is designed with $g(x)=\lambda \|x\|_1$ example in mind.  In this case problem \eqref{prob:g-shrink} takes a form of an unconstrained Lasso problem \cite{Tibshirani_1996}. 
We consider matrices $H$ which are a sum of a diagonal matrix and a low-rank matrix and we  apply randomized coordinate descent to solve \eqref{prob:g-shrink} approximately. An extension to the group sparsity term $g(x)=\lambda\sum \|x_i\|_2$ \cite{Qin2010}, is rather straightforward.

Problems of the form (\ref{prob:P}) with $g(x)=\lambda \|x\|_1$ have been the focus of much research lately in the fields of signal processing and machine learning. This form encompasses a variety of machine learning models, 
in which feature selection is desirable, such as sparse logistic regression \cite{Yuan2010,nGLMNET,shalev2009stochastic}, sparse inverse covariance selection \cite{Hsieh2011,Olsen2012,Sinco2009} and unconstrained Lasso \cite{Tibshirani_1996}, etc. These settings often present common difficulties to optimization algorithms due to their large scale. During the past decade most optimization effort aimed at these problems focused on development of efficient first-order methods, such as accelerated proximal gradient methods \cite{Nesterov,Beck2009,Sparsa}, block coordinate descent methods \cite{nGLMNET,GLMNET,glasso_2008,Sinco2009} and alternating directions methods \cite{Alm_Scheinberg}. These methods enjoy low per-iteration complexity, but typically have slow local convergence rates. Their performance is often hampered by small step sizes. This, of course, has been known about first-oder methods for a long time, however, due to the very large size of these problems, second-order methods are often not a practical alternative.
In particular, constructing and storing a Hessian matrix, let alone inverting it, is prohibitively expensive for values of $n$ larger than $10000$, which often makes the use of the Hessian in large-scale problems prohibitive, regardless of the benefits of fast local convergence rate. 

Nevertheless, recently several new methods were proposed for sparse optimization which make careful use of second-order information \cite{Hsieh2011,nGLMNET,Olsen2012,Chin2012}. 
These new methods are designed to exploit the special structure of the Hessian of  specific functions to improve  efficiency of  solving \eqref{prob:g-shrink}. Several successful methods employ coordinate descent to approximately solve   \eqref{prob:g-shrink}. While other approaches to solve Lasso subproblem were considered in 
\cite{Chin2012},  none generally outperform coordinate descent, which is well suited when special structure of the Hessian approximation, $H$,
 can be exploited and when low accuracy of the subproblem solutions is sufficient. 
In particular, \cite{nGLMNET} proposes a specialized GLMNET \cite{GLMNET} implementation for sparse logistic regression, where coordinate descent method is applied to the unconstrained Lasso subproblem constructed using the Hessian of $f(x)$. The special structure of the Hessian is used to reduce the complexity cost of each coordinate step so that it is linear in the number of training instances, and a two-level shrinking scheme proposed to focus the minimization on  smaller subproblems. Similar ideas are used in \cite{Hsieh2011} in a specialized algorithm called QUIC for sparse inverse covariance selection, where the Hessian of $f(x)$ also has a favorable structure for solving Lasso subproblems.
Another specialized method for graphical Markov random fields  was recently proposed in \cite{icml2013_wytock13}. This method also exploits special Hessian structure to improve coordinate descent efficiency. 

There are other common features shared by the  methods described above. These methods are often referred to as proximal Newton-type methods.
 The overall algorithmic framework can be described as follows:
\begin{itemize}
\item At each  iteration $k$ the smooth function $f(x)$ is approximated near the current iterate $x^k$ by a convex quadratic function $q^k(x)$. 
\item A working subset of coordinates (elements) of $x$ is selected for subproblem optimization. 
\item Then $l(k)$ passes of coordinate descent are applied   to optimize (approximately) 
the function $q^k(x)+g(x)$ over the working set, which results in a trial point. Here $l(k)$ is some linear function of $k$.
\item The trial point is  accepted as the new iterate if it satisfies some sufficient decrease condition (to be specified).
\item Otherwise,  a line search is applied to compute a new trial point. 
\end{itemize}

In this paper we {\em do not } include the theoretical analysis of various working set selection strategies. Some of these have been analyzed in the prior literature (e.g., see \cite{LewisWright11}). Combining such existing analysis with the rate of convergence results in this paper is a subject of a future study. 

This paper contains the following  three main results. 
\begin{enumerate}
\item We discuss   theoretical properties of the above framework in terms of global convergence rates.
In particular, we show that if we replace the line search by a prox-parameter update mechanism, we can derive sublinear global convergence results for the above methods under mild assumptions on Hessian approximation matrices, which can include diagonal, quasi-Newton and limited memory quasi-Newton approximations. We also provide the convergence rate for the case of inexact subproblem optimization. 
It turns out that standard global convergence analysis of proximal gradient methods (see \cite{Beck2009,Schmidtetal}) does not extend in a natural way to  proximal
quasi-Newton frameworks, hence we use a different technique derived for smooth optimization in \cite{Nesterov, NesterovPolyak, Cartisetal2012},  in a novel way, to obtain the global complexity result.  

\item The heuristic of applying $l(k)$ passes of coordinate descent to the subproblem is very useful in practice, but has not yet been theoretically justified, due to the lack of known complexity estimates. Here we use probabilistic complexity bounds of randomized coordinate descent to show that this heuristic is indeed  well justified theoretically.
In particular, it guarantees the sufficiently rapid decrease of the expectation of the error in the subproblems  and hence allows for sublinear global convergence rate to hold for the entire algorithm (again, in expectation). This gives us the  first complete global convergence rate result for the algorithmic schemes for  practical (inexact) proximal Newton-type methods. Moreover, using the new analysis from \cite{NesterovConvexBook2004, NesterovPolyak, Cartisetal2012} we are able to provide lower overall complexity
bound than the one that follows from  \cite{Schmidtetal}
\item
Finally, we propose an efficient {\em general purpose} algorithm that uses the same theoretical framework, but  which does not rely on the special structure of the Hessian, and yet in our tests compares favorably with the state-of-the-art, specialized methods such as   QUIC and GLMNET. We replace the exact Hessian computation by the limited memory BFGS Hessian approximations  \cite{NoceWrig06}  (LBFGS) and exploit their special structure within a coordinate descent approach to solve the subproblems.
\end{enumerate}

Let us elaborate a bit further on the new approaches and results developed in this paper and discuss related prior work. 

In \cite{Byrdetal2013} Byrd et al. propose that the methods in the framework  described above should be referred to as sequential quadratic approximation (SQA) instead of proximal Newton methods.  They reason that there is no proximal operator or proximal term involved in this framework. This is indeed the case, if a line search is used to ensure sufficient decrease. Here we propose to consider a prox term as a part of the quadratic approximation. 
Instead of a line search procedure, we update the prox term of our quadratic model, which allows us to extend global convergence bounds of proximal gradient methods to the case of proximal (quasi-)Newton methods. The criteria for accepting a new iteration is based on sufficient decrease condition (much like in trust region methods, and unlike that in proximal gradient methods).
We show that our mechanism of updating the prox parameter, based on sufficient decrease condition, leads to an improvement in performance and robustness of the algorithm compared to the line search approach as well as enabling us to develop global convergence rates. 

Convergence results for the proximal Newton method have been shown 
in \cite{Saundersetal} and more recently in \cite{Byrdetal2013} (with the same sufficient decrease condition as ours, but applied within a line search). These papers also demonstrate super linear local convergence rate of the proximal Newton and a proximal quasi-Newton method.
Thus it is confirmed in \cite{Byrdetal2013, Saundersetal} that using second order information is as beneficial for problems of the form  \eqref{prob:P}
as it is for the smooth optimization problems.  These results apply to our framework when exact Hessian (or a quasi-Newton approximation) of $f(x)$ is used to construct $q(x)$ (and if the matrices have bounded eigenvalues).  However, theory in \cite{Byrdetal2013, Saundersetal}
does not provide global convergence rates for these methods, and  just as in the case on smooth optimization,
the super linear local convergence rates, generally,  do not apply in the case of LBFGS Hessian approximations. 

The convergence rate that we show is sub linear, which is generally the best that can be expected from a proximal (quasi-)Newton method
 with no assumptions on the accuracy of the Hessian approximations.  Practical benefits of using 
LBFGS Hessian approximations is well known for smooth optimization \cite{NoceWrig06} and have been exploited in many large scale applications. 
In this paper we demonstrate this benefit in the composite optimization setting \eqref{prob:P}. Some prior work showing benefit of limiter memory quasi-Newton method in proximal setting include \cite{BeckerNIPS2012, Jiangetal2012}. 
We also emphasize  in our theoretical analysis the potential gain over proximal gradient methods in terms of constants occurring in the convergence rate. 

To prove the sub linear rate  we borrow a technique from  \cite{NesterovConvexBook2004, NesterovPolyak, Cartisetal2012}. The technique used in 
  \cite{Beck2009} and \cite{Schmidtetal} for the proof of convergence rates of the (inexact) proximal gradient method do not seem to extend
 to  general positive definite Hessian approximation matrix.  In a  related work \cite{Jiangetal2012}  the authors analyze global convergence rates of an accelerated proximal quasi-Newton method, as an extension of FISTA method \cite{Beck2009}. The convergence rate they obtain match that of accelerated proximal gradient methods, hence it is a faster rate
than that of our method presented here. However,  they have to impose much stricter conditions on the Hessian approximation matrix, in particular they require  that the difference between any two consecutive Hessian approximations (i.e., $H_{k}-H_{k+1}$) is positive semidefinite. This is actually in contradiction to   FISTA's  requirement  that the prox parameter is never increased. Such a condition is very restrictive 
as is impractical. In this paper we briefly show how  results in \cite{Beck2009} and \cite{Schmidtetal}  can be extended under some 
(also possibly strong) assumptions of the Hessian approximations, to give a simple and natural convergence rate analysis. 
We then present an alternative analysis, which only requires the Hessian approximations to have bounded eigenvalues. 
Investigating accelerated version of our approach without restrictive assumptions on the Hessian approximations and with the use 
of randomized coordinate descent is a subject of future research.

 Finally, we use the complexity analysis of randomized coordinate descent in \cite{Richtarik2012} to provide a simple and efficient stopping criterion for the subproblems and thus derive the total complexity of  proximal (quasi-)Newton methods based on randomized coordinate descent to solve Lasso subproblems.

The paper is organized as follows: in Section \ref{sec:basic}  we describe the  algorithmic framework. Then,  in Section \ref{sec:conv_basic}  we 
present some of the assumptions and discussions and, in Section \ref{sec:inexact_prox},  convergence rate analysis based on \cite{Beck2009} and \cite{Schmidtetal}. 
In  Section \ref{sec:complete_conv} we show the new convergence rate analysis using  \cite{NesterovConvexBook2004, NesterovPolyak, Cartisetal2012} for exact and inexact
version of our framework.  We then extend the analysis for the cases where inexact solution to a subproblem is random in Section \ref{sec:random}
and in particular to the randomized coordinate descent in Section \ref{sec:coordinate_descent_iteration_complexity}. 
Brief description of the details of our  proposed algorithm are in Section \ref{sec:alg} and  computational results  validating the theory are  presented in Section \ref{sec:comp}.


\section{Basic algorithmic framework and theoretical analysis}\label{sec:basic} 

The following function is used throughout the paper as an approximation of the objective function $F(x)$.

\bea \label{def-Qf}
Q(H,u,v) := f(v) +
\langle \nabla f(v),  u-v \rangle + \frac{1}{2} \langle u-v, H(u-v)\rangle + g(u).
\eea 


%
For a fixed point $\bar x$, the function $Q(H,  x, \bar x)$ serves as an approximation of $F(x)$ around  $\bar x$. 
Matrix $H$ controls the quality of this approximation. In particular, if $f(x)$ is smooth  and  $H=\frac{1}{\mu}I$, then $Q(H, x , \bar x)$ is a sum of the prox-gradient approximation of $f(x)$ at $\bar x$ and $g(x)$. This particular form of $H$ plays a key role in the design and analysis of proximal gradient methods (e.g., see  \cite{Beck2009}) 
and alternating direction augmented Lagrangian methods (e.g, see \cite{Alm_Scheinberg}).
 If  $H=\nabla^2 f(\bar x)$, then  $Q(H, x ,\bar x)$
is a second order approximation of $F(x)$  \cite{Saundersetal, SchmidtQP}.
In this paper we assume that  $H$ is a positive definite matrix such that $\sigma I \preceq H\preceq M I$ for some positive constants $M$ and $\sigma$. 

Minimizing the function $Q(H, u,v)$ over $u$ reduces to solving problem \eqref{prob:g-shrink}. We will use the following notation to denote the accurate and approximate solutions of \eqref{prob:g-shrink}.
\begin{equation}\label{eq:pv}
p_H(v):=\arg\min_u Q(H, u,v),  
\end{equation}
and 
\begin{equation}\label{eq:pvphi}
p_{H, \phi}(v) {\rm \ is \ a\ vector\ such\ that:\ } \  \begin{array}{l}  Q(H, p_{H, \phi}(v),v)\leq Q(H, v,v)=F(v),\ {\rm and }\\
Q(H, p_{H, \phi}(v),v)\leq Q(H, p_{H}(v),v)+\phi. \end{array}
\end{equation}

The method that we consider in this paper computes iterates by (approximately) optimizing $Q(H, u,v)$ with respect to $u$ using some particular  $H$  which is chosen at each iteration.
 The basic algorithm is described in Algorithms \ref{alg:ISTA-SD} and \ref{alg:Backtrack_SD}. 
\begin{algorithm2e}\caption{Proximal Quasi-Newton method}
    \label{alg:ISTA-SD}%
{\rm Choose }
$0<\rho\leq 1$ and  $x^0$\; 
\For{$k=0,1,2,\cdots$}{
 Choose $0<\bar \mu_k,  \phi_k>0, G_k\succeq 0$\;
Find  $H_k=G_k+\frac{1}{2\mu_k}I$ and  $x^{k+1}  := p_{H_k}(x^k)$  \\ by applying {\em Prox\ Parameter\ Update\ }$(\bar \mu_k, G_k, x^k, \rho)$\;
}
\end{algorithm2e}

\begin{algorithm2e}\caption{Prox Parameter Update $(\bar \mu, G, x, \rho)$ }
    \label{alg:Backtrack_SD}%
Select $0<\beta<1$ and set $\mu=\bar \mu$\; \For{$i=1,2,\cdots$}{
Define $H=G+\frac{1}{2\mu} I$ and compute $p(x):=p_{H}(x)$\;
If $F(p(x))- F(x) \leq \rho (Q(H,  p(x),x)- F(x))$, then output $H$ and $p(x)$, {\bf Exit} \;
Else $\mu=\beta^{i}\bar  \mu$\;
}
\end{algorithm2e}

Algorithm \ref{alg:Backtrack_SD}  chooses Hessian approximations of the form $H_k=\frac{1}{\mu_k}I +G_k$. However, it is possible to consider any procedure of choosing positive definite $H_k$ which ensures $M I \succeq H_k\succeq \sigma I$ and 
$F(p_{H_k}(x))- F(x) \leq \rho (Q(H_k, p_{H_k}(x),x)- F(x))$, for a given $0<\rho\leq1$, - a step acceptance condition which is a relaxation of conditions used in  \cite{Beck2009} and \cite{Schmidtetal}.

An inexact version of Algorithm \ref{alg:ISTA-SD}  is obtained by simply replacing $p_{H_k}$ by $p_{H_k, \phi_k}$ in both Algorithms  \ref{alg:ISTA-SD} and \ref{alg:Backtrack_SD} for some sequence of $\phi_k$ values.  

\section{Basic results, assumptions and preliminary analysis}
\label{sec:conv_basic}

 ISTA  \cite{Beck2009} is a particular case of Algorithm \ref{alg:ISTA-SD} with $G_k=0$, for all $k$, and $\rho=1$. 
In this case, the value of  $\mu_k$ is chosen so that the conditions of Lemma \ref{lem:subp_inexact} hold with $\epsilon=0$. In other words, the reduction achieved in the objective function $F(x)$ is at least the amount of reduction achieved in the model 
$ Q(\mu_k,p(x^k),x^k)$. It is well known that as long as $\mu_k\leq 1/L(f)$ (recall that $L(f)$ is the Lipschitz constant of the gradient) then  condition \eqref{eq:dec_cond}
 holds with $\rho=1$.  Relaxing condition \eqref{eq:dec_cond} by using $\rho<1$ allows us to accept larger values of $\mu_k$, which in turn implies larger steps taken by the algorithm.   This basic idea is the cornerstone of step size selection in most nonlinear optimization algorithms. Instead of insisting on achieving "full" predicted reduction of the objective function (even when possible), a fraction of this reduction is usually sufficient. In our experiments small values of $\rho$ provided much better performance than values close to $1$.

In the next three sections we present the analysis of convergence rate of  Algorithm \ref{alg:ISTA-SD} under different scenarios. Recall that we assume that $f(x)$ is convex and smooth, in other words $\|\nabla f(x)-\nabla f(y)\|\leq L(f)\|x-y\|$ for all $x$ and $y$ in the domain of interest, while $g(x)$ is simply convex. In Section \ref{sec:coordinate_descent_iteration_complexity} we assume that $g(x)=\lambda \|x\|_1$. Note that we do not assume that $f(x)$ is strongly convex or that it is twice continuously differentiable, because we do not rely on any accurate second order information in our framework. We only assume that the Hessian approximations are positive definite and bounded, but their accuracy can be arbitrary, as long as sufficient decrease condition holds. Hence we only achieve sublinear rate of convergence. To achieve higher local rates of convergence stronger assumptions on $f(x)$ and on the Hessian approximations have to be made, see for instance, \cite{Byrdetal2013} and \cite{Saundersetal} for related local convergence analysis. 

First we present a helpful lemma which is a simple extension of Lemma 2 in \cite{Schmidtetal} to the case of general positive definite Hessian estimate. This lemma establishes some simple properties of an $\phi$-optimal solution to the proximal problem (\ref{prob:g-shrink}).  It uses the concept of the $\phi$-subdifferential of a convex function $a$ at $x$, 
$\partial_{\phi} a(x)$, which is defined as the set of vectors $y$ such that $a(x) - y^Tx \leq a(t) - y^Tt + \phi$ for all $t$. 
\begin{lemma}
    \label{lem:inexact_1st_opt_cond}
    Given $\phi >0$, a p.d. matrix $H$ and $v\in \br^n$, let $p_{\phi}(v)$ denote the $\phi$-optimal solution to the proximal problem (\ref{prob:g-shrink}) in the sense that
    \begin{align}
        \label{equ:eps_minimizer}
        g(p_{\phi}(v))+\frac{1}{2}\|p_{\phi}(v)-z\|_H^2 \leq \phi +
        \min_{x\in\br^n}\left\{g(x)+\frac{1}{2}\|x-z\|_H^2 \right\},
    \end{align}
    where $z = v - H^{-1} \nabla f(v)$. Then there exists $\eta$ such that $\frac{1}{2} \| \eta \|^2_{H^{-1}} \leq \phi$ and
    \begin{align}
        \label{equ:inexact_solvQ_lem}
        H(v-p_{\phi}(v))  - \eta \in \partial_{\phi}g(p_{\phi}(v)).
    \end{align}
\end{lemma}

\begin{proof}
    (\ref{equ:eps_minimizer}) indicates that $p_{\phi}(v)$ is an $\phi$-minimizer of the convex function $a(x) := \frac{1}{2}\|x-z\|_H^2+g(x)$. If we let $a_1(x) = \frac{1}{2}\|x-z\|_H^2$ and $a_2(x) = g(x)$, then this is equivalent to
    \begin{align}
        \label{equ:zero_in_subdiff}
        0 \subset \partial_{\phi} a(p_{\phi}(v)) \subset \partial_{\phi} a_1(p_{\phi}(v)) + \partial_{\phi} a_2(p_{\phi}(v)).
    \end{align}
    Hence,
    \begin{align*}
        \partial_{\phi} a_1(p_{\phi}(v)) &= \left\{ y \in \br^n ~|~ \frac{1}{2} \| y + H(z-p_{\phi}(v)) \|^2_{H^{-1}} \leq \phi \right\} \\
        &= \left\{ y \in \br^n, ~y = \eta - H(z-p_{\phi}(v)) ~|~ \frac{1}{2} \|\eta\|^2_{H^{-1}} \leq \phi \right\}.
    \end{align*}
    From (\ref{equ:zero_in_subdiff}) we have
    \begin{align}
        H(z-p_{\phi}(v)) - \eta \in \partial_{\phi}g(p_{\phi}(v)) \mbox{ with } \frac{1}{2} \|\eta\|^2_{H^{-1}} \leq \phi.
    \end{align}
    Then (\ref{equ:inexact_solvQ_lem}) follows using $z = v - H^{-1} \nabla f(v)$.
\end{proof}

We will find the following bound useful on the norm of $\eta$ which follows from the above lemma.
\begin{equation}\label{eq:etabound}
\|\eta_i\|\leq \sqrt{2\lambda_{max}(H)\phi_i},
\end{equation}
where $\lambda_{max}$ is the largest eigenvalue of $H$ or its upper bound. 

Below are the assumptions made in our analysis.

\begin{assumptions}$\\$ 
\label{as:exact_conv_rate}
	\assume
	\label{assub:optimal_exist}
	The set of optimal solutions of \eqref{prob:P}, $X^*$, is nonempty and $x^*$ is any element of that set. $\\$
	\assume
	\label{assub:level_set}
	The effective domain of $F$ is defined as $\dom(F) := \{ x \in \Rmbb^n: F(x) < \infty\}$, and the level set of $F$ at point $x \in \dom(F)$ is defined by 
	\begin{align*}
	    \Xcal_F(x) := \{y \in \dom(F): F(y) \leq F(x)\}.
	\end{align*}
	Without loss of generality, we restrict our discussions below to the level set $\Xcal_0 := \Xcal_F(x^0)$ given by some $x^0 \in \dom(F)$, e.g., the initial iterate of the Algorithm \ref{alg:ISTA-SD}.
	$\\$
	\assume 
	\label{assub:gLipsch} 
	$g$ is convex and Lipschitz continuous with constant $L_g$ for all $x, y \in \Xcal_0$:
	\begin{align*}
	    g(x)-g(y)\leq L_g\|x-y\|, 
	\end{align*}
	\assume 
	\label{assub:bound_h}
	There exists positive constants $M$ and $\sigma$ such that for all $k\geq 0$, at the $k$-th iteration of Algorithm \ref{alg:ISTA-SD}:
	\begin{align}
	     \sigma I \preceq \sigma_k I \preceq H_k \preceq M_k I \preceq MI
	\end{align} 
	\assume
	\label{assub:bound_level_set}
	There exists a positive constant $D_{\Xcal_0}$ such that for all iterates $\{x^k\}$ of Algorithm \ref{alg:ISTA-SD}:
	\begin{align*}
	     \sup_{x^* \in X^*}~\|x^k - x^*\| \leq D_{\Xcal_0}
	\end{align*} 
\end{assumptions}

In the analysis of the proximal gradient methods  Assumption~\ref{assub:bound_level_set} is removed by directly establishing a uniform bound on $\|x^k - x^*\|$ (rf. \cite{OML,Schmidtetal}).  In the next section we show an outline of a proximal-gradient type analysis where this assumption is imposed  for simplicity of the presentation. It is possible \cite{OML} to establish a similar, but more complex bound without this assumption.
The proximal gradient approach in the next subsection, however, requires another, stronger, assumption on $H_k$. In the alternative analysis that follows 
we impose Assumption~\ref{assub:bound_level_set} but relax the assumption on $H_k$ matrices. 
Note that all iterates $\{x^k\}$ fall into the level set $\Xcal_0$, due to the sufficient decrease condition \eqref{eq:dec_cond} that demands a monotonic decrease on the objective values. Assumption~\ref{assub:bound_level_set} thus follows straightforwardly if the level set is bounded, which often holds  in real-world problems or can be easily imposed.

\subsection{Analysis via proximal gradient approach}\label{sec:inexact_prox}
In this section we extend the analysis in \cite{Beck2009} and  \cite{Schmidtetal} to our framework under additional assumptions on the Hessian approximations $H_k$.
The following lemma,  is a  generalization of Lemma 2.3 in \cite{Beck2009} and of a similar lemma in \cite{Schmidtetal}.
This lemma serves to provide a bound on the change in the objective function $F(x)$. 
\begin{lemma}
    \label{lem:subp_inexact}
    Given $\epsilon$, $\phi$ and $H$ such that
    \begin{align}
        \label{lem:BT2.3-assump-g}
        F(p_{\phi}(v)) &\leq Q(H, p_{\phi}(v),v) +\epsilon \\
        Q(H, p_{\phi}(v),v) &\leq \min_{x\in\br^n} Q(H,  x,v) + \phi, \nonumber
    \end{align}
    where $p_{\phi}(v)$ is the $\phi$-approximate minimizer of $Q(H, x,v)$, then for any  $\eta$ such that $\frac{1}{2} \| \eta \|^2_{H^{-1}} \leq \phi$
    and for any $u\in \br^n$,
    \begin{align}
        2(F(u) - F(p_{\phi}( v))) \geq \|p_{\phi}( v)-u\|_H^2 - \|v-u\|_H^2-2\epsilon-2\phi-2 \langle \eta, u - p_{\phi}(v) \rangle. \nonumber
    \end{align}
\end{lemma}

\vskip5mm

\begin{proof}   
    The proof is an easy extension of  that in \cite{Beck2009}.
\end{proof}

\vskip5mm

Note that if $\phi=0$, that is the subproblems are solved accurately, then we have $2(F(u) - F(p( v))) \geq \|p( v)-u\|_H^2 - \|v-u\|_H^2-2\epsilon$.

If the sufficient decrease  condition 
\bea\label{eq:dec_cond}
(F(x^{k+1})-F(x^k))\leq \rho (Q(H_k, x^{k+1}, x^k)- F(x^k))
\eea
is satisfied, then
\begin{align*}
    F(x^{k+1})
    &\leq Q(H_k, x^{k+1}, x^k)  - (1-\rho) \left( Q(H_k, x^{k+1}, x^k)- F(x^k) \right) \\
    &\leq Q(H_k,  x^{k+1}, x^k)  - \frac{1-\rho}{\rho} (F(x^{k+1})- F(x^k))
\end{align*}
and  Lemma \ref{lem:subp_inexact} holds at each iteration $k$ of Algorithm \ref{alg:ISTA-SD} with  $\epsilon_k=-\frac{1-\rho}{\rho}(F(x^{k+1})- F(x^k))$. 

We now  establish the sub linear convergence rate of Algorithm \ref{alg:ISTA-SD} under the following additional assumption. 

\begin{assumption}\label{as:strong_matrices}
Let $\{x^k\}$ be the sequence of iterates generated by Algorithm \ref{alg:ISTA-SD}, then there exists a constant $M_H$ such that 
 \begin{equation}\label{eq:strong_matrices}
 \sum_{i=0}^{k-1}(\|x^{i+1} - x^*\|^2_\frac{H_{i+1}}{M_{i+1}}-\|x^{i+1} - x^*\|^2_ \frac{H_i}{M_i})\leq M_H, \quad \forall k,
 \end{equation}
 where $M_i$ is the upper bound on the largest eigenvalue of $H_i$ as defined in Assumption \ref{assub:bound_h}. 
 \end{assumption}
 
 The assumption above is not verifiable, hence it does not appear to be very useful. However, it is easy to see that a condition 
 $\frac{H_{i+1}}{M_{i+1}}\preceq \frac{H_i}{M_i}$, for all $i$, easily implies \eqref{eq:strong_matrices} with $M_H=0$. This condition, in turn,
  is trivially satisfied if $H_{i+1}$
 is a multiple of $H_i$ for all $i$, as it is in the case of ISTA algorithm. It is also clear that Assumption \ref{as:strong_matrices} is a lot weaker than
 the condition that $H_{i+1}$ is a multiple of $H_i$ for all $i$. For example, it also hold if $H_{i+1}$ is a multiple of $H_i$ for all $i$ {\em except for a finite number of iterations}. It also holds if $H_{i+1}$ converges to $H_i$ (or to its multiple) sufficiently rapidly. It is also weaker than the assumption
 in \cite{Jiangetal2012} that $H_{i+1}\preceq {H_i}$. Exploring different conditions on the Hessian approximations that ensure Assumption \ref{as:strong_matrices}  is a subject of a separate study. Below we show how under this condition sub linear convergence rate is established. 

  \begin{theorem}\label{th:inexact_conv_rate}
 Suppose that Assumptions~\ref{as:exact_conv_rate} and~\ref{as:strong_matrices} hold. 
 Assume that all iterates $\{x^k\}$ of inexact Algorithm \ref{alg:ISTA-SD} are generated with some $\phi_k\geq 0$ (cf. \eqref{eq:pvphi}), then
 \begin{align}
     \label{eq:bound_F_F*}
     F(x^k) - F(x^*) \leq 
     \frac{1}{k_m} \left( \frac{1}{2} \|x^0-x^*\|_{H_0/M_0}^2 
     + M_H+ \frac{(1-\rho)(F(x^0)-F(x^*))}{\rho\sigma} 
     + \sum_{i=0}^{k-1}\frac{\phi_i}{M_i}
     + D_{\Xcal_0}\sum_{i=0}^{k-1} \sqrt{\frac{2\phi_i}{M_i}} \right),
 \end{align}
 where $k_m=\sum_{i=0}^{k-1}M_i^{-1}$.
 \end{theorem}

 \begin{proof}
     Let us apply Lemma \ref{lem:subp_inexact}, sequentially, with $u=x^*$, $p_{\phi}(v)=x^i$ and subproblem minimization residual $\phi_i$ for $i=0, \ldots, k-1$. Adding up resulting inequalities, we obtain  
     \begin{align}
         &\sum_{i=0}^{k-1} \frac{F(x^{i+1}) - F(x^*)}{M_i}\leq 
         \frac{1}{2} \sum_{i=0}^{k-1} \left(\|x^i-x^*\|_{H_i/M_i}^2 
         - \|x^{i+1}-x^*\|_{H_i/M_i}^2 \right) 
         + \sum_{i=0}^{k-1}( \frac{\phi_i}{M_i} + \frac{\epsilon_i}{M_i} +\frac{\langle \eta_i,x^* - x^{i+1} \rangle}{M_i})\nonumber \\
         \label{equ:bound_F_diff}
         &= 
          \frac{1}{2} \left( \|x^0-x^*\|_{\frac{H_0}{M_0}}^2 
         + \sum_{i=0}^{k-1}(\|x^{i+1} - x^*\|^2_{H_{i+1}/M_{i+1}}-\|x^{i+1} - x^*\|^2_{H_i/M_i})
         - \|x^k-x^*\|_{H_k/M_k}^2 \right) \\
         &\nonumber + \sum_{i=0}^{k-1} \frac{\epsilon_i}{M_i} 
         + \sum_{i=0}^{k-1} \frac{\phi_i}{M_i} 
         + \sum_{i=0}^{k-1} \frac{\langle \eta_i,x^* - x^{i+1} \rangle}{M_i}.
         \end{align}
         From Assumptions~\ref{as:exact_conv_rate} and~\ref{as:strong_matrices} we have
         \begin{align*}
         &\sum_{i=0}^{k-1} \frac{F(x^{i+1}) - F(x^*)}{M_i} \leq 
         \frac{1}{2} \|x^0-x^*\|_{H_0/M_0}^2 +M_H
         + \sum_{i=0}^{k-1} \frac{\epsilon_i}{M_i} 
         + \sum_{i=0}^{k-1} \frac{\phi_i}{M_i} 
         + D_{\Xcal_0}\sum_{i=0}^{k-1} \frac{\|\eta_i\|}{M_i}. 
         \end{align*}
         From \eqref{eq:etabound}  and definition of $M_i$ we know that $\|\eta_i\|\leq \sqrt{2M_i\phi_i}$. 
         Using the already established bound $\sum_{i=0}^{k-1}\epsilon_i \leq \frac{(1-\rho)(F(x^0)-F(x^*))}{\rho}$ we obtain
          \begin{align*}
         &\sum_{i=0}^{k-1} \frac{F(x^{i+1}) - F(x^*)}{M_i}\leq
         \frac{1}{2} \|x^0-x^*\|_{H_0/M_0}^2 +M_H
         + \frac{(1-\rho)(F(x^0)-F(x^*))}{\rho\sigma} 
         + \sum_{i=0}^{k-1}\frac{\phi_i}{M_i}
         + D_{\Xcal_0}\sum_{i=0}^{k-1} \sqrt{\frac{2\phi_i}{M_i}}.
     \end{align*}
     Hence,
     \begin{align*}
         F(x^k) - F(x^*) &\leq \frac{1}{k_m}\sum_{i=0}^{k-1} \frac{F(x^{i+1}) - F(x^*)}{M_i}\\        
         &\leq \frac{1}{k_m} \left( \frac{1}{2} \|x^0-x^*\|_{H_0/M_0}^2 +M_H
         + \frac{(1-\rho)(F(x^0)-F(x^*))}{\rho\sigma} 
         + \sum_{i=0}^{k-1}\frac{\phi_i}{M_i}
         + D_{\Xcal_0}\sum_{i=0}^{k-1} \sqrt{\frac{2\phi_i}{M_i}} \right).
     \end{align*}
 \end{proof}

Let us consider the term $\frac{1}{k_m}=\frac{1}{\sum_{i=0}^{k-1}M_i^{-1}}$. From earlier discussions, we can see that 
$\frac{1}{k_m}\leq\frac{M}{k}$. Moreover, if $H_k$ are diagonal matrices, then $M=L(f)$ and, hence $\frac{1}{k_m}\leq\frac{L(f)}{k}$,
which established a bound similar to that of proximal gradient methods. The role of $M_i$ is to show that if most of these values are much smaller than the global Lipschitz constant $L(f)$, then the constant involved in the sub linear rate can be much smaller than that of the proximal gradient methods.
This is  well known effect of using partial second order information and it is observed in our computational results. 

We conclude that under Assumption~\ref{as:strong_matrices} Algorithm \ref{alg:ISTA-SD} converges at the rate of $O(1/k)$ if 
$\sum_{i=0}^{k-1} \sqrt{\frac{2\phi_i}{M_i}}$ is bounded for all $k$. This result in similar to those obtained in \cite{Schmidtetal}.
In \cite{OML} it is shown how randomized block coordinate descent and other methods can be utilized to optimize subproblems 
$\min_u Q(H_i,u, x^i)$
so that  $\sqrt{\frac{2\phi_i}{M_i}}$ decays sufficiently fast to guarantee  such a bound (possibly in expectation).

In this paper, however, we focus on a different derivation of the sub linear convergence rate, which results in a different bound on $\phi_k$
and different, more complex, dependence on the constants, but, on the other hand, does not require Assumption~\ref{as:strong_matrices}
and results in a weaker assumption on $\phi_i$. 

\section{Analysis of sub linear convergence}\label{sec:complete_conv}
In our analysis below we will use another known technique for establishing sub linear convergence of  gradient descent type methods
on smooth convex functions \cite{Cartisetal2012}. However, due to the non smooth nature of our function the analysis requires significant extensions,
especially in the inexact case. Moreover, it does not apply to the line-search algorithm, we will rely on the fact that a proximal quasi-Newton method is used in that each new iteration $x^{k+1}$ is an approximate minimizer of the function $Q(H_k, u, x^k)$. Our analysis, hence, also applies to proximal gradient methods. 

First we prove the following simple result.
\begin{lemma}
Consider $F(\cdot)$ defined in \eqref{prob:P}. Let Assumptions~\ref{assub:gLipsch} hold. Then for any three points $u, v, w \in \dom(F)$, we have
\begin{align}
	\label{equ:bound_F_u_w_inexact}
    F(u) - F(w) 
    \leq 
    \|\nabla f(u)+\gamma_{g, \phi}^v\|\|u-w\| + 2L_g \|u-v\| + 2\phi.
\end{align}
where $\gamma_{g, \phi}^v \in \partial_{\phi} g(v)$ is any $\phi$-subgradient of $g(\cdot)$ at point $v$.
\end{lemma}
\begin{proof}
	From convexity of $f$ and $g$ and the definition of $\phi$-subgradient, it follows that for any  points $u, w$ and $v$,
	\begin{align}
		\label{equ:_f_convex_upper}
	    f(u) - f(w) &\leq \langle \nabla f(u), u-w \rangle, \\
	    \label{equ:_g_convex_upper_inexact}
	    g(v) - g(w) &\leq \langle \gamma_{g, \phi}^v, v-w \rangle + \phi\\
	     \label{equ:_g_convex_upper_inexact2}
	    g(2v-u) - g(v) &\geq \langle \gamma_{g, \phi}^v, v-u \rangle- \phi.
	\end{align}
	Hence,
	\begin{align}
	    \nonumber
	    F(u)-F(w)
	    &=f(u)-f(w)+g(u)-g(w) \\
	    \nonumber
	    &=f(u)-f(w)+g(u)-g(v)+g(v)-g(w) \\
	    \label{proof:f_g_upper}
	    &\leq \langle \nabla f(u), u-w \rangle
	    +\langle \gamma_{g, \phi}^v, v-w \rangle 
	    +g(u)-g(v) 
	    + \phi\\
	    \nonumber
	    &= \langle \nabla f(u), u-w \rangle
	    +\langle \gamma_{g, \phi}^v, u-w \rangle 
	    +\langle \gamma_{g, \phi}^v, v-u \rangle 
	    +g(u)-g(v)
	    + \phi\\
	    \nonumber
	   & \leq \|\nabla f(u)+\gamma_{g, \phi}^v\|\|u-w\|  +
	  	    g(2v-u) - g(v)+g(u)-g(v) 
	    +2 \phi. 
	\end{align}
	Here  we applied \eqref{equ:_f_convex_upper}, \eqref{equ:_g_convex_upper_inexact} 
	and \eqref{equ:_g_convex_upper_inexact2} to get 
	 \eqref{proof:f_g_upper}. 
	Using 
	  Assumption~\ref{assub:gLipsch} to bound the term $g(2v-u) - g(v)+g(u)-g(v)$ 
	  we can easily derive \eqref{equ:bound_F_u_w_inexact}. 
\end{proof}

\subsection{The exact case}
We now consider the exact version of Algorithm \ref{alg:ISTA-SD}, i.e., $\phi_k=0$ for all $k$. 
We have the following lemma. 
\begin{lemma}\label{lem:bound_step_size}
Let $x^{k+1}:=\arg\min_{u\in\br^n} Q(H_k, u,x_k)$, then 
\begin{align}
	\label{equ:bound_Qk_Qk1}
    Q(H_k,x^k,x^k) - Q(H_k,x^{k+1}, x^k) \geq \frac{\sigma_k}{2}\|x^{k+1}-x^k\|^2.
\end{align}
Moreover, there exists a vector $\gamma_g^{k+1} \in \partial g(x^{k+1})$ such that the following bounds hold: 
\begin{align}
    \label{equ:bound_step_size}
    \frac{1}{M_k}\|  \nabla f(x^k) + \gamma_{g}^{k+1} \|
    \leq 
    \|x^{k+1} - x^k\| 
    \leq 
    \frac{1}{\sigma_k}\|  \nabla f(x^k) + \gamma_{g}^{k+1} \|.
\end{align}
\end{lemma}
\begin{proof}
     The proof is a special case of Lemma~\ref{lem:bound_step_size_inexact}, proved below.
\end{proof}

\begin{theorem}\label{the:ISTA-SD}
Let Assumptions \ref{as:exact_conv_rate} hold for all $k$.
Then the iterates $\{x^k\}$ generated by Algorithm \ref{alg:ISTA-SD} satisfy
\begin{align}
    \label{the:ISTA-nonsmooth-conclude} 
    F(x^k)-F^* \leq  \frac{2M^2 (D_{\Xcal_0} M+2L_g)^2}{\rho\sigma^3}\frac{1}{k}.
\end{align}
\end{theorem}
\begin{proof}
	We will denote $F(x^k)-F^*$ by $\Delta F_k$.  Our goal is to bound $\Delta F_k$ from above  in terms of $1/k$ which
	 we will achieve by deriving a lower bound  on $\frac{1}{\Delta F_k}$ in terms of $k$. 

	Let us first show that 
	\begin{equation}\label{eq:mainbnd}
	F(x^{k})-F(x^{k+1})=\Delta F_k-\Delta F_{k+1}\geq c_k \Delta F_k^2,
	\end{equation}
	for some constant $c_k$, which depends on iteration $k$, but will be lower bounded by a uniform constant. 

	First we will show that 
	\begin{equation}\label{eq:lowerbnd}
	\Delta F_k\leq (D_{\Xcal_0}+\frac{2L_g}{\sigma_k}) \|\nabla f(x^k)+\gamma_g^{k+1}\|.
	\end{equation}
	This follows simply from \eqref{equ:bound_F_u_w_inexact} with $u = x^k, w = x^*, v = x^{k+1}$ and $\phi = 0$,
	\begin{align}
	\label{equ:bound_delta_F_k}
	    \Delta F_k = F(x^k) - F(x^*)
	    \leq \|\nabla f(x^k)+\gamma_{g}^{k+1}\|\|x^k-x^*\| + 2L_g \|x^k-x^{k+1}\|. 
	\end{align}
	Substituting the bounds on  $\| x^{k+1} - x^k\|$ (cf. \eqref{equ:bound_step_size}) and   $\| x^{k} - x^*\|$ (cf. Assumptions~\ref{assub:bound_level_set})
	in \eqref{equ:bound_delta_F_k} we get
	the desired bound \eqref{eq:lowerbnd}.  

	Now we will show that 
	\begin{equation}\label{eq:upperbnd}
		F(x^k) - F(x^{k+1})\geq \frac{\rho\sigma_k}{2M_k^2}\|\nabla f(x^k)+\gamma_g^{k+1}\|^2.
	\end{equation}
	Indeed $F(x^{k})-F(x^{k+1})\geq \rho (Q(H_k, x^{k}, x^k)-Q(H_k,x^{k+1}, x^k))$. And a bound on the reduction in $Q$ can be established by combining \eqref{equ:bound_Qk_Qk1} and \eqref{equ:bound_step_size},
    \begin{align*}
	    &Q(H_k,x^k,x^k) - Q(H_k,x^{k+1}, x^k) \geq \frac{\sigma_k}{2M_k^2}   \| \nabla f(x^k) + \gamma_g^{k+1}\| ^2,  
    \end{align*}
	Hence, the bound \eqref{eq:upperbnd} holds.

	Finally, combining the lower bound on $F(x^{k+1})-F(x^k)$ together with the upper bound on $\Delta F_k^2$  we can conclude that
	\[
	F(x^{k+1})-F(x^k)=\Delta F_{k+1}-\Delta F_k\geq  \frac{\rho\sigma_k}{2M_k^2 (D_{\Xcal_0}+\frac{2L_g}{\sigma_k})^2}\Delta F_k^2,
	\]
	which establishes \eqref{eq:mainbnd} with $c_k= \frac{\rho\sigma_k^3}{2M_k^2 (D_{\Xcal_0}\sigma_k+2L_g)^2}$. 

	Dividing both sides of the inequality above  by $\Delta F_{k+1}\Delta F_k$ we have
	\[
	\frac{1}{\Delta F_{k}}-\frac{1}{\Delta F_{k+1}}\geq c_k\frac{\Delta F_{k}}{\Delta F_{k+1}}\geq c_k.
	\]
	Summing the above expression for  $i=0, \ldots, k-1$ we have
	\[
	\frac{1}{\Delta F_{k}}\geq \sum_{i=0}^{k-1}c_i+\frac{1}{\Delta F_{0}} \geq \sum_{i=0}^{k-1}c_i,
	\]
	which finally implies
	\[
	\Delta F_k=F(x^k)-F^*\leq \frac{1}{\sum_{i=0}^{k=1} c_i} \leq  \frac{2M^2 (D_{\Xcal_0}M+2L_g)^2}{\rho\sigma^3}\frac{1}{k}.
	\]
\end{proof}

Let us note that if $H_k=L(f)I$ for all $k$, as in standard proximal gradient methods, where $L(f)$ is the Lipschitz constant of $\nabla f(x)$, then
the bound becomes
\[
F(x^k)-F^*\leq   \frac{ 2(D_{\Xcal_0}L(f)+2L_g)^2}{\rho L(f)}\frac{1}{k}\approx \frac{ 2D_{\Xcal_0}^2L(f)}{k},
\]
if $L_g\ll D_{\Xcal_0}L(f)$. This  bound is similar to $\frac{2\|x^0-x^*\|^2L(f)}{k}$ established for proximal gradient methods, assuming that $D_{\Xcal_0}$ is comparable to
$\|x^0-x^*\|$.

\subsection{The inexact case} 
\label{sec:conv_inexact}

We now  analyze Algorithm \ref{alg:ISTA-SD}  in the case when the computation of $p_H( v)$ is performed inexactly. 
In other words, we consider the version of Algorithm \ref{alg:ISTA-SD} (and \ref{alg:Backtrack_SD}) where we compute 
$x^{k+1}  := p_{H_k, \phi_k}(x^k)$ and $\phi_k$ can be positive for any $k$. The analysis is similar to that of the exact case, with a few additional terms that need to be bounded. We begin by extending Lemma \ref{lem:bound_step_size}.

\begin{lemma}\label{lem:bound_step_size_inexact}
Let $x^{k+1} \in \dom(F)$ be such that the inequality $Q(H_k, x^{k+1},x^k)-\min_{y\in\br^n} Q(H_k, y,x^k)\leq \phi_k$ holds 
for some $\phi_k \geq 0$. Then 
\begin{align}
	\label{equ:bound_Qk_Qk1_inexact}
    Q(H_k,x^k,x^k) - Q(H_k,x^{k+1}, x^k) \geq \frac{\sigma_k}{2}\|x^{k+1}-x^k\|^2-\sqrt{2M_k\phi_k}\|x^{k+1}-x^k\| - \phi_k.
\end{align}
Moreover there exists a vector $\gamma_{g, \phi}^{k+1} \in \partial g_{\phi_k}(x^{k+1})$ such that the following bounds  hold: 
\begin{align}
    \label{equ:bound_step_size_inexact}
    \frac{1}{M_k}\|  \nabla f(x^k) + \gamma_{g, \phi}^{k+1} \|- \frac{\sqrt{2M_k\phi_k}}{M_k} 
    \leq 
    \|x^{k+1} - x^k\| 
    \leq 
    \frac{1}{\sigma_k}\|  \nabla f(x^k) + \gamma_{g, \phi}^{k+1} \| + \frac{\sqrt{2M_k\phi_k}}{\sigma_k}.
\end{align}
\end{lemma}
\begin{proof}
Recall Lemma \ref{lem:inexact_1st_opt_cond}, from \eqref{equ:inexact_solvQ_lem}, there
 exists a vector, which we will refer to as $\gamma_{g, \phi}^{k+1}$, such that
    \begin{align}       \label{eq:gammagdefine}
        \gamma_{g, \phi}^{k+1}=H_k(x^k-x^{k+1})- \nabla f(x^k)  - \eta_k \in \partial_{\phi}g(x^{k+1}),
    \end{align}
    with  $\frac{1}{2}\|\eta_k\|^2_{H_k^{-1}}\leq \phi_k$, which, in turn, implies $\|\eta_k\| \leq \sqrt{2M_k\phi_k}$.
  The following inequality follows from the definition of $\phi$-subdifferential,
	\begin{align}
        \label{equ:_g_phi_subgradient}
        g(x^k) - g(x^{k+1}) 
        \geq 
        \langle \gamma_{g, \phi}^{k+1}, x^k-x^{k+1} \rangle - \phi_k. 
    \end{align}
    From  \eqref{eq:gammagdefine}
    \begin{align}
        \label{equ:_bound_Q_Qstar_1st_optima}
        H_k(x^{k+1}-x^k) + \nabla f(x^k) + \gamma_{g, \phi}^{k+1}+\eta_k= 0, 
    \end{align}
     hence we obtain
    \begin{align*}
		\| x^{k+1} - x^k\|
		&\leq  \frac{1}{\sigma_k}\|  \nabla f(x^k) + \gamma_{g, \phi}^{k+1} \| + \frac{1}{\sigma_k}\|\eta_k\|
		\leq  
		\frac{1}{\sigma_k}\|  \nabla f(x^k) + \gamma_{g, \phi}^{k+1} \| + \frac{\sqrt{2M_k\phi_k}}{\sigma_k}, \\
		\| x^{k+1} - x^k\|
		&\geq   \frac{1}{M_k}\|  \nabla f(x^k) + \gamma_{g, \phi}^{k+1} \|-   \frac{1}{M_k}\|\eta_k\|
		\geq   
		\frac{1}{M_k}\|  \nabla f(x^k) + \gamma_{g, \phi}^{k+1} \|- \frac{\sqrt{2M_k\phi_k}}{M_k}.
	\end{align*}
	Using \eqref{equ:_g_phi_subgradient} and \eqref{equ:_bound_Q_Qstar_1st_optima}, we derive \eqref{equ:bound_Qk_Qk1_inexact} as follows,
	\begin{align*}
	        &Q(H_k,x^k,x^k) - Q(H_k,x^{k+1}, x^k) \\
	        &= g(x^k) - g(x^{k+1}) 
	        - \nabla f(x^k)^T( x^{k+1}-x^k)  
	        - \frac{1}{2}\|x^{k+1}-x^k\|^2_{H_k}  \\
	        &\geq \langle \gamma_{g, \phi}^{k+1}, x^{k}-x^{k+1} \rangle 
	        - \phi_k
	        + \|x^{k+1}-x^k\|_{H_K}^2 
	        +  \langle \gamma_{g, \phi}^{k+1}, x^{k+1} - x^{k}\rangle
	        + \langle \eta_k, x^{k+1}-x^k \rangle
	        - \frac{1}{2}\|x^{k+1}-x^k\|^2_{H_k} \\
			&= \frac{1}{2}\|x^{k+1}-x^k\|^2_{H_k} 
			+ \langle \eta_k,  x^{k+1}-x^k \rangle - \phi_k
			\geq \frac{\sigma_k}{2}\|x^{k+1}-x^k\|^2
			- \|\eta_k\|\|x^{k+1}-x^k\| - \phi_k  \\
	        &\geq \frac{\sigma_k}{2}\|x^{k+1}-x^k\|^2-\sqrt{2M_k\phi_k}\|x^{k+1}-x^k\| - \phi_k.
	\end{align*}
\end{proof}

Unlike the exact case, the inquality 
$\frac{1}{\Delta F_{k+1}}-\frac{1}{\Delta F_{k}}\geq c_k$ can no longer be guaranteed to hold on each iteration. 
The convergence rate is obtained by observing that when this desired inequality fails another inequality always holds, which bounds 
$\Delta F_{k}$ in terms of $\phi_k$. Specifically, we have the following theorem.
\begin{lemma}
\label{lem:bound_iter_inexact}
	Consider $k$th iteration of the inexact Algorithm \ref{alg:ISTA-SD} with $0 \leq \phi_k\leq 1$. Let $\Delta F_k := F(x^k) - F(x^*)$. Then there exists large enough positive constant $\theta > 0$, such that one of the following two cases must hold,
	\begin{align}
		\label{equ:bound_F_inexact_case1}
	    &\Delta F_k \leq b_k\sqrt{\phi_k}, \\
	    \label{equ:bound_F_inexact_case2}
	    &\frac{1}{\Delta F_{k+1}}-\frac{1}{\Delta F_{k}}\geq c_k.
	\end{align}
	where $b_k$ and $c_k$ are given below,
	\begin{align}
	\label{equ:bound_bk_ck}
	    b_k = \theta D_{\Xcal_0}\sqrt{2M_k}+  \frac{2(1+\theta)L_g}{\sigma_k}\sqrt{2M_k} + 2, 
	    ~
	    c_k=\frac{\rho(\sigma_k^3(\theta-1)^2 - 2\sigma_kM_k^2(1+\theta) - \sigma_k^3M_k )}{(\sqrt{2}D_{\Xcal_0}\theta\sigma_kM_k + 2\sqrt{2}L_g(1+\theta)M_k + \sigma_k\sqrt{M_k} )^2}.
	\end{align}
\end{lemma}
\begin{proof}	      
First, applying \eqref{equ:bound_F_u_w_inexact} with $u = x^k, w = x^*$ and $v = x^{k+1}$  we obtain,
	\begin{align}
	\label{equ:_bound_F_xk_xk1}
	    \Delta F_k = F(x^k) - F(x^*)
	    \leq \|\nabla f(x^k)+\gamma_{g, \phi}^{k+1}\|\|x^k-x^*\| + 2L_g \|x^k-x^{k+1}\| + 2\phi_k. 
	\end{align}          
	
	We will consider two cases that are possible at each iteration $k$ for some fixed constant $\theta>1$ which we will specify later. 
	\begin{itemize}
	\item{Case 1} 
	\begin{equation}\label{eq:case1}
		\|  \nabla f(x^k) + \gamma_{g, \phi}^{k+1} \| < \theta\sqrt{2M_k\phi_k}.
	\end{equation}

	\item{Case 2}
	\begin{equation}\label{eq:case2}
		\| \nabla f(x^k) + \gamma_{g, \phi}^{k+1} \| \geq \theta \sqrt{2M_k\phi_k}.
	\end{equation}

	\end{itemize}

	Let us assume that Case 1 holds, then from \eqref{equ:bound_step_size_inexact}  and \eqref{eq:case1} it simply 
	follows that
	\begin{equation}\label{eq:bndcase1} 
		\| x^{k+1} - x^k\|\leq \frac{(1+\theta)\sqrt{2M_k\phi_k}}{\sigma_k}.
	\end{equation}

	Using \eqref{eq:case1}, \eqref{equ:_bound_F_xk_xk1}, the bound on  $\| x^{k+1} - x^k\|$ from \eqref{eq:bndcase1} together with the bound on $\|x^{k} - x^*\|$ from Assumptions~\ref{assub:bound_level_set} we get
	\begin{align}\label{eq:phi_k_replaced}
	    \Delta F_k
	    &\leq D_{\Xcal_0} \|\nabla f(x^k)+\gamma_{g, \phi}^{k+1}\| +  \frac{2(1+\theta)L_g\sqrt{2M_k\phi_k}}{\sigma_k} +2 \phi_k \\
	    \nonumber &\leq (\theta D_{\Xcal_0}+  \frac{2(1+\theta)L_g}{\sigma_k})\sqrt{2M_k\phi_k} + 2\phi_k 
	    \leq (\theta D_{\Xcal_0}\sqrt{2M_k}+  \frac{2(1+\theta)L_g}{\sigma_k}\sqrt{2M_k} + 2) \sqrt{\phi_k},
	\end{align}
	  which ensures \eqref{equ:bound_F_inexact_case1}. 
	  
	We now consider Case 2, where \eqref{eq:case2} along with \eqref{equ:bound_step_size_inexact} from Lemma  \ref{lem:bound_step_size_inexact} imply
	\begin{equation}\label{eq:bndcase2} 
		\frac{\theta-1}{\theta M_k}\| \nabla f(x^k) + \gamma_{g, \phi}^{k+1} \| \leq \| x^{k+1} -
		 x^k\|\leq \frac{1+\theta}{\theta\sigma_k}\| \nabla f(x^k) + \gamma_{g, \phi}^{k+1} \|.
	\end{equation}

	Substituting into \eqref{equ:_bound_F_xk_xk1} the upper bound on  $\| x^{k+1} - x^k\|$ from \eqref{eq:bndcase2} and the bound on $\|x^{k} - x^*\|$ from Assumptions~\ref{assub:bound_level_set} we now get  
	\begin{align*}
		\Delta F_k\leq 
		(D_{\Xcal_0} + \frac{2L_g(1+\theta)}{\theta\sigma_k})\|\nabla f(x^k)+\gamma_{g, \phi}^{k+1}\| + 2\phi_k.
	\end{align*}
	From $ \phi_k \leq 1$ it follows that $\| \nabla f(x^k) + \gamma_{g, \phi}^{k+1} \| \geq \theta \sqrt{2M_k\phi_k} \geq \theta \sqrt{2M_k}\phi_k$. Hence we obtain
	\begin{align}
	\label{eq:upperbndinexact}
		\Delta F_k\leq 
		(D_{\Xcal_0} + \frac{2L_g(1+\theta)}{\theta\sigma_k} + \frac{2}{\theta\sqrt{2M_k}})\|\nabla f(x^k)+\gamma_{g, \phi}^{k+1}\|.
	\end{align}
	 
	We will  show that in this case, as in the exact case,  we have 
	\begin{equation}\label{eq:mainbnd}
		\Delta F_k - \Delta F_{k+1} = 
		 F(x^k) - F(x^{k+1})\geq c_k \Delta F_k^2
	\end{equation} for some constant $c_k$ (different from that in the exact case).  
	Towards that goal  we will establish a lower bound on $F(x^{k})-F(x^{k+1})$ in terms of $\|\nabla f(x^k)+\gamma_{g, \phi}^{k+1}\|^2$ as before.
	We still have $F(x^{k})-F(x^{k+1})\geq \rho (Q(H_k, x^{k}, x^k)-Q(H_k,x^{k+1}, x^k))$. 
	We now use the bounds \eqref{equ:bound_Qk_Qk1_inexact} from Lemma  \ref{lem:bound_step_size_inexact}, \eqref{eq:case2} and \eqref{eq:bndcase2} and obtain
	\begin{align*}
		&Q(H_k,x^k,x^k) - Q(H_k,x^{k+1}, x^k) \\
		&\geq \frac{\sigma_k}{2}(\frac{\theta-1}{\theta M_k})^2 \| \nabla f(x^k) + \gamma_{g, \phi}^{k+1} \|^2
		- \frac{1+\theta}{\theta^2\sigma_k}\| \nabla f(x^k) + \gamma_{g, \phi}^{k+1} \|^2 
		- \frac{1}{2\theta^2M_k}\| \nabla f(x^k) + \gamma_{g, \phi}^{k+1} \|^2
	\end{align*}
	By selecting a sufficiently large $\theta$  we can ensure that 
	\begin{align}
	\label{eq:lowerbndinexact}
	    F(x^k) - F(x^{k+1})\geq \rho t_k \| \nabla f(x^k) + \gamma_{g, \phi}^{k+1} \|^2,
	\end{align}
	for $t_k=\frac{\sigma_k}{2}(\frac{\theta-1}{\theta M_k})^2 - \frac{1+\theta}{\theta^2\sigma_k}- \frac{1}{2\theta^2M_k}>0$. 
	Finally, combining the lower bound \eqref{eq:lowerbndinexact} on $F(x^{k+1})-F(x^k)$ together with the upper bound \eqref{eq:upperbndinexact}
	on $\Delta F_k^2$  we can conclude that \eqref{eq:mainbnd} holds with 
	\begin{align*}
		c_k= \frac{\rho t_k}{(D_{\Xcal_0} + \frac{2L_g(1+\theta)}{\theta\sigma_k} + \frac{2}{\theta\sqrt{2M_k}})^2} = \frac{\rho(\sigma_k^3(\theta-1)^2 - 2\sigma_kM_k^2(1+\theta) - \sigma_k^3M_k )}{(\sqrt{2}D_{\Xcal_0}\theta\sigma_kM_k + 2\sqrt{2}L_g(1+\theta)M_k + 2\sigma_k\sqrt{M_k} )^2}.    
	\end{align*}
	Finally, \eqref{equ:bound_F_inexact_case2} follows from \eqref{eq:mainbnd} divided by $\Delta F_k\Delta F_{k+1}$ and using the fact that $\frac{\Delta F_k}{\Delta F_{k+1}} \geq 1$.
	
\end{proof}

\begin{remark} Let us discuss the result of the above lemma. The lemma applies for any value of $\theta$ for which
$t_k$, and hence, $c_k$ is positive for all $k$.  It is easy to see that large values of  $\theta$ 
imply large values of $c_k$. On the other hand, large $\theta$ is likely to cause  Case 1 to hold (i.e., \eqref{equ:bound_F_inexact_case1}) 
instead of Case 2 (i.e., \eqref{equ:bound_F_inexact_case2}) on any given iteration, with $b_k$ also growing with the size of $\theta$. 
As we will show below the overall rate of convergence of the algorithm is derived using the two bounds - \eqref{equ:bound_F_inexact_case1}, where
the rate is controlled by the rate of $\phi_k\to 0$ and \eqref{equ:bound_F_inexact_case2}, which is similar to the bound in the exact case.
The overall bound, thus, will depend on the upper bound on $b_k$'s and the inverse of 
the lower bound on $c_k$'s. If, again, we assume that $\sigma_k=M_k=L(f)$ for all $k$, then $\theta=O(\sqrt{L(f)})$ is sufficient to ensure that $c_k>0$
and this results in $b_k\leq O(D_{\Xcal_0}L(f))$ and $1/c_k\geq O(D_{\Xcal_0}^2L(f))$, thus again, we obtain a bound which is comparable to that of proximal gradient methods, although with more complex constants. 
\end{remark}

We now derive the overall convergence rate, under the assumption that $\phi_{k}$ decays sufficiently fast. 

\begin{theorem}\label{th:inexact_conv_rate}
Suppose that Assumption~\ref{as:exact_conv_rate} holds. 
Assume that all iterates $\{x^k\}$ of inexact Algorithm \ref{alg:ISTA-SD} are generated with some $\phi_k\geq 0$ that satisfy
\begin{align}
\label{equ:phi_conv_rate_required}
    \phi_k \leq \frac{a^2}{k^2}, \mbox{ with } 0<a \leq 1. 
\end{align}
Let  $\theta$ be chosen as specified in Lemma~\ref{lem:bound_iter_inexact}. Then for any $k$
\begin{align}
    \label{eq:bound_F_F*}
	F(x^k) - F(x^*) 
    \leq \frac{\max\{ba, \frac{1}{c}\}}{k-1}
\end{align}
with $b, c$  given as follows,
\begin{align}
\label{equ:bound_b_c}
    b = \theta D_{\Xcal_0}\sqrt{2M}+  \frac{2(1+\theta)L_g}{\sigma}\sqrt{2M} + 2, 
    ~
    c = \frac{\rho \left( \frac{\sigma^2}{M^2}(1 - \theta^{-1})^2 - (2\theta^{-1} + 3\theta^{-2}) \right)}{2(D_{\Xcal_0} +
     \frac{2L_g(1+\theta^{-1})}{\sqrt{\sigma}} + \frac{2}{\sqrt{2}}\theta^{-1})^2}.
\end{align}
\end{theorem}
\begin{proof}
Consider all iterations before a particular iteration $k$.  From Lemma~\ref{lem:bound_iter_inexact}, it follows that either 
 \eqref{equ:bound_F_inexact_case1} or \eqref{equ:bound_F_inexact_case2} must hold for each prior iteration.
       Let $k_1 <k$  denote the index of the last iteration before $k$,  for which \eqref{equ:bound_F_inexact_case1} holds. If no such $k_1$ exists, then \eqref{equ:bound_F_inexact_case2} holds for all $k$ and without loss of generality we can consider $k_1=0$. From \eqref{equ:bound_F_inexact_case1} and from the fact that the function value never increases
	 \begin{align}
	\label{equ:_thm_conv_rate_case1}
	    \Delta F_{k_1+1}=F(x^{k_1+1}) - F(x^*) \leq   F(x^{k_1}) - F(x^*) =\Delta F_{k_1} \leq b_{k_1} \sqrt{\phi_{k_1}k_1} \leq  \frac {a b_{k_1}}{k_1}
	\end{align}
	which is the same as
		 \begin{align}\label{equ:_thm_conv_rate_Fk1_upper}
	  \frac{1}{\Delta F_{k_1+1}} \geq  \frac{k_1} {a b_{k_1}}.
	  \end{align}
	  For each iteration from  $k_1+1$ to $k-1$, \eqref{equ:bound_F_inexact_case2}  gives 
	\begin{align*}
	    \frac{1}{\Delta F_{k_1 + i + 1}} - \frac{1}{\Delta F_{k_1 + i}} \geq c_i, ~\forall i = 1, ..., k-k_1-1
	\end{align*}
	Summing up the above inequalities  and using \eqref{equ:_thm_conv_rate_Fk1_upper} we obtain the following bound on 
	$\Delta F_{k}$, 
	\begin{align}\label{eq:finalsum}
	    \frac{1}{\Delta F_{k }} \geq \sum_{i=1}^{k-k_1-1}c_i+  \frac{1}{\Delta F_{k_1 + 1}} \geq  
	    \sum_{i=1}^{k-k_1-1}c_i+ \frac{k_1} {a b_{k_1} }.
	\end{align}
	
	To derive  a simple uniform bound on  $\Delta F_{k }$ we will use  $b$ and $c$ - uniform upper and lower bounds, respectively, for $b_k$ and $c_k$ given in \eqref{equ:bound_bk_ck}, i.e., $b_k \leq b$, $c_k \geq c, \forall k\geq k_0$. 
	From Assumptions~\ref{assub:bound_h} we can derive the expressions for $c$ as follows,
	\begin{align*}
	    c_k
	    &= \frac{\rho \left( \frac{\sigma_k}{2}(\frac{\theta-1}{\theta M_k})^2 - \frac{1+\theta}{\theta^2\sigma_k}- \frac{1}{2\theta^2M_k} \right)}{(D_{\Xcal_0} + \frac{2L_g(1+\theta)}{\theta\sigma_k} + \frac{2}{\theta\sqrt{2M_k}})^2} 
	    \geq \frac{\rho \left( \frac{\sigma}{2}(\frac{\theta-1}{\theta M})^2 - \frac{1+\theta}{\theta^2\sigma}- \frac{1}{2\theta^2\sigma} \right)}{(D_{\Xcal_0} + \frac{2L_g(1+\theta)}{\theta\sigma} + \frac{2}{\theta\sqrt{2\sigma}})^2} 
	    \geq \frac{\rho \left( \frac{\sigma^2}{2M^2}(\frac{\theta-1}{\theta})^2 - \frac{1+\theta}{\theta^2}- \frac{1}{2\theta^2} \right)}{(D_{\Xcal_0} + \frac{2L_g(1+\theta)}{\theta\sqrt{\sigma}} + \frac{2}{\theta\sqrt{2}})^2} 
	\end{align*}
	Bound $b$ can be obtained in a similar fashion.

Substituting bounds $b$ and $c$ in \eqref{eq:finalsum} we get 
\begin{align*}
	    \frac{1}{\Delta F_{k }} \geq (k-k_1-1)c+ \frac{k_1} {a b}\geq \min\{c, \frac{1} {a b}\}(k-1),
	    \end{align*}
 which is the same as   \eqref{eq:bound_F_F*}.


\end{proof}

\begin{remark} It follows that the inexact version of Algorithm \ref{alg:ISTA-SD}   has sublinear convergence rate if 
$\phi_i\leq a^2/i^2$  for some $a<1$ and all iterations $i=0, \ldots, k$. 
In contrast,  the bounds in \cite{Schmidtetal} and in Section \ref{sec:inexact_prox} 
require that $\sum_{i=0}^{\infty} \sqrt{\phi_i}$ is  bounded. 
This bound  on the overall sequence is clearly 
stronger than $\phi_i\leq a^2/i^2$, since $\sum_{i=0}^{\infty} \frac{a}{i}=\infty$. 
On the other hand, it does not impose any particular requirement on any given iteration, except that 
each $\phi_i$ is finite, which our bound on $\phi_i$ is assumed to hold at each iteration, so far. 
\end{remark}

\subsection{Complexity in terms of subproblem solver iterations}
Let us discuss  conditions on the sequence of $\phi_i$ established above and how they can be ensured. 
Firstly, let us note that condition $a<1$ in Theorem \ref{th:inexact_conv_rate} can easily be removed. 
We introduced it for the sake of brevity, to ensure that $\phi_i\leq 1$ on each iteration. Clearly, an arbitrarily large $a$ can be used and in that case
$\phi_i\leq 1$ for all $i\geq 1/a$. Moreover, the condition $\phi_i\leq 1$ is only needed to replace $\phi_i$ with $\sqrt{\phi_i}$ in 
Lemma \ref{lem:bound_iter_inexact} in inequality \eqref{eq:phi_k_replaced}. Instead we can use bound \eqref{eq:case2} and 
upper bounds on $\nabla f(x^i)$ and $\gamma^{i+1}_{g,\phi}$ to replace $\phi_k$ with a constant multiple of $\sqrt{\phi_i}$. 
In conclusion, it is sufficient to solve the subproblem on the $i$-th iteration to accuracy $O(1/i^2)$.

%
The question now is: what method and what stopping criterion should be used for subproblem optimization, so that sufficient accuracy is achieved 
and no excessive computations are performed, in other word, how can we guarantee the bound on $\phi_i$,  
while maintaining efficiency of the subproblem optimization? 
It is possible to consider terminating
 the optimization of the $i$-th subproblem once the duality gap is smaller than the required bound on $\phi_i$.  
 However, checking duality gap can be computationally  very expensive. 
 Alternatively one can use an algorithm with a known  convergence rate. 
 This way it can be determined apriori how many iterations of such an algorithm should be applied to the $i$-th subproblem to achieve the desired accuracy. In particular, we note that the objective functions in our subproblems  are all $\sigma$-strongly convex, so a simple proximal gradient method, or some of its accelerated versions,  will enjoy linear convergence rates when applied to these subproblems.  Hence, after $i$ iterations of optimizing
  $Q_i$,  such a method will achieve accuracy $\phi_i$ that decays geometrically, i.e., $\phi_i=C\delta^i$, for some constants $C>0$ and $0<\delta<1$, hence 
$\sum_{i=0}^{\infty} \sqrt{\phi_i}$ is  bounded.   Note that the same property holds for any linearly convergent  method, such as the proximal gradient or a semi-smooth Newton method. 
 Also, it is easy to see that $\phi_i\leq a^2/i^2$ holds for some $a>0$ for all $i$.   
One can also can use FISTA \cite{Beck2009} to optimize $Q_i$ which will ensure 
$\phi_i\leq a^2/i^2$ for some $a>0$ but will not guarantee $\sum_{i=0}^{\infty} \sqrt{\phi_i}$.
The advantage of using FISTA and its resulting rate is that ist does not depend on the strong convexity constant,
 hence the subproblem complexity does 
not depend on  $\sigma$ - the lower bound on the smallest eigenvalues of the Hessian approximations.
 In conclusion, we have the following overall bounds.

\begin{theorem}\label{th:inexact_total_conv_rate_k2}
Suppose that Assumptions~\ref{as:exact_conv_rate} hold and that at the $k$-th iteration of inexact Algorithm \ref{alg:ISTA-SD} 
function $Q(H_k, u, x^k)$ is approximately minimized, to obtain $x^{k+1}$ by applying $l(k)=\alpha k+\beta$ steps of any algorithm which guarantees that
 $Q(H_k, x^{k+1}, x^k)\leq Q(H_k, x^k, x^k)$ and whose convergence 
rate ensures  the error bound $\phi_k\leq a^2/(\alpha k+\beta)^2$ for some $a>0$.  Then accuracy $F(x^k) - F(x^*)\leq \epsilon$ is achieved after at most 
\begin{align}
    \label{eq:total_bound_F_F*}
	K=\beta( { \frac{\max\{b a, \frac{1}{c}\}}{\epsilon}} +1)+ \frac{\alpha}{2}({ \frac{\max\{b a, \frac{1}{c}\}}{\epsilon}})({ \frac{\max\{b a, \frac{1}{c}\}}{\epsilon}}+1)
\end{align}
inner iterations (of the chosen algorithm), 
with $b, c$  given in Theorem \ref{th:inexact_conv_rate}. 
\end{theorem}
\proof{A proof follows trivially from Theorem  \ref{th:inexact_conv_rate}. }
\begin{theorem}\label{th:inexact_total_conv_rate_klogk}
Suppose that Assumptions~\ref{as:exact_conv_rate} hold and that at the $k$-th iteration of inexact Algorithm \ref{alg:ISTA-SD} 
function $Q(H_k, u, x^k)$ is approximately minimized, to obtain $x^{k+1}$ by applying  $l(k)$ steps of an algorithm, which guarantees that
 $Q(H_k, x^{k+1}, x^k)\leq Q(H_k, x^k, x^k)$ and whose convergence 
rate ensures  the error bound $\phi_k\leq \delta^{l(k)}M_Q$, for some constants $0<\delta<1$ and $M_Q>0$. 
 Then, by setting $l_k= 2 log_{\frac{1}{\delta}}(k)$, accuracy $F(x^k) - F(x^*)\leq \epsilon$ is achieved after at most 
\begin{align}
    \label{eq:total_bound_F_F*}
	K=\sum_{k=0}^t 2 \log_{\frac{1}{\delta}}(k)\leq 2t \log_{\frac{1}{\delta}}(t)
	\end{align}
	inner iterations (of the chosen algorithm), with $t=\lceil \frac{\max\{b a, \frac{1}{c}\}}{\epsilon} +1\rceil$
and  $b, c$  given in Theorem \ref{th:inexact_conv_rate}. 
\end{theorem}
\proof{A proof follows trivially from Theorem  \ref{th:inexact_conv_rate}. }

\begin{remark} The total complexity in terms of the inner iterations should not be viewed as  a summary of the whole complexity of Algorithm \ref{alg:ISTA-SD}.
A key step of the algorithm is the computation of $F(x^{k})$ and $\nabla f(x^k)$ at each iteration. In  big data applications this is often the 
most extensive step, hence the main complexity is defined by the number of function and gradient computations. 
Due to backtracking via proximal parameter update to satisfy sufficient decrease condition, 
the number of function and gradient  computation may be larger than the number of iterations of 
Algorithm \ref{alg:ISTA-SD}, however,
 it does not exceed this number  by more than a logarithmic factor. In practice, only several initial iterations contain backtracking steps, hence 
 Theorem  \ref{th:inexact_conv_rate} provides the bound  on the complexity in terms of function and gradient computations. 
\end{remark}

In the next section we extend our convergence rate results to the case of solving subproblems via randomized coordinate descent,  where
 $\phi$ is random and hence does not satisfy required bounds on each iteration. 

\section{Analysis of the inexact case under random subproblem accuracy}\label{sec:random}
As we pointed out in the introduction, the most efficient  practical approach to subproblem optimization, in the case when  $g(x)=\lambda \|x\|_1$, seems to be the coordinate descent method. One iteration of a coordinate descent step can be a lot less expensive 
than that of a proximal gradient or a Newton method. In particular, if matrix $H$ is constructed via the 
LBFGS approach, then one step of a coordinate decent
method takes a constant number 
of operations, $m$ (the memory size of LBFGS, which is typically 10-20). 
On the other hand, one step of proximal gradient takes $O(mn)$ operations and Newton method takes $O(nm^2)$. 

Unfortunately,  cyclic (Gauss-Seidel) coordinate descent, does not have deterministic complexity bounds, hence it is not possible to know when the work on a particular subproblem can be terminated to guarantee the desired  level of accuracy. However, a randomized coordinate descent has probabilistic complexity bounds, which can be used to demonstrate the linear rate of convergence in expectation.

We have the following probabilistic extension of Theorem \ref{th:inexact_conv_rate}.

\begin{theorem}\label{th:inexact_random_conv_rate}
Suppose that Assumption~\ref{as:exact_conv_rate} holds. 
Assume that for all $k$ iterates $\{x^k\}$ of inexact Algorithm \ref{alg:ISTA-SD} are generated with some $\phi_k\geq 0$ that satisfy
\begin{align}
\label{equ:phi_conv_rate_required}
    P\{\phi_k \leq \frac{a^2}{k^2}\}\geq 1-p, \mbox{ for\ some\ } 0<a \leq 1 \mbox{ and }  0\leq p<1,
\end{align}
conditioned on the past. 
Let  $\theta$, $b$ and $c$ be as specified in Theorem \ref{th:inexact_conv_rate}. Then for any $k$
\begin{align}
    \label{eq:bound_F_F*}
	E(F(x^k) - F(x^*) )
    \leq \frac{\max\{ba, \frac{1}{c}\}(2-p)}{(1-p)(k-1)}. 
\end{align}
\end{theorem}
\begin{proof}
As in  the proof of Theorem \ref{th:inexact_conv_rate} consider all iterations before a particular iteration $k$.  From Lemma~\ref{lem:bound_iter_inexact}, it follows that either \eqref{equ:bound_F_inexact_case1} or \eqref{equ:bound_F_inexact_case2} must hold for each prior iteration. 
       Let $k_1 <k$  denote the index of the last iteration before $k$,  for which \eqref{equ:bound_F_inexact_case1} holds. Let $k_2$ denote the index of the second to last such iteration and so on, hence $k_i$ is the index of the last iteration such that there are exactly $i$ iterations between $k_i$ and 
       $k-1$ for which \eqref{equ:bound_F_inexact_case1} holds. Without loss of generality we can assume that $k_i$ exists for each $i$, because if it does not - we can set $k_i=0$ and obtain a better bound. 
       Let us now assume for a given $i$ that $\phi_{k_i} \leq \frac{a^2}{{k_i}^2}$ holds, but that $\phi_{k_j} > \frac{a^2}{{k_j}^2}$ for all $j=1,\ldots, i-1$ (if $i=1$ we have the case analyzed in the proof of Theorem \ref{th:inexact_conv_rate}). The analysis in the proof of Theorem \ref{th:inexact_conv_rate} extends 
       easily to the case  $i>1$ by observing that 
       \begin{align*}
	    \frac{1}{\Delta F_{l+1}} - \frac{1}{\Delta F_{l}} \geq c_l
	\end{align*}
	holds for any $k_i+1\leq l\leq k-1$, $l\neq k_j,\, j=1\ldots, i-1$, that is all the iterations for which  or \eqref{equ:bound_F_inexact_case1} does not hold. 
	We also have
	 \begin{align*}
	    \frac{1}{\Delta F_{l+1}} - \frac{1}{\Delta F_{l}} \geq 0, 
	\end{align*}
	for $k_i+1\leq l\leq k-1$, $l= k_j,\, j=1\ldots, i-1$, simply from the fact that function values never increase. 
	Summing up the above inequalities  and using the fact that $\phi_{k_i} \leq \frac{a^2}{{k_i}^2}$ we obtain 
	$\Delta F_{k}$, 
	\begin{align}\label{eq:finalsum}
	    \frac{1}{\Delta F_{k }} \geq    
	    \sum_{l=1}^{k-k_i-i}c_l+ \frac{k_l} {a b_{k_l} }\geq   (k-k_i-i)c+ \frac{k_i} {a b },
	\end{align}
	and finally, we have 
		\begin{align}\label{eq:finalsum}
	    \Delta F_{k } \leq    \frac{\max\{ba, \frac{1}{c}\}}{k-i}. 
	\end{align}

	Now, recall that  $P\{\phi_{k_i} \leq \frac{a^2}{{k_i}^2}\}\geq 1-p$, for any iteration $i$, independently of the other iteration. 
	This means that the probability that 
	$\{\phi_{k_i} \leq \frac{a^2}{{k_i}^2}\}$ and $\{\phi_{k_j} > \frac{a^2}{{k_j}^2}\}$ for all $j=1\ldots, i-1$ is $(1-p)p^i$. 
	  This implies
		\begin{align}\label{eq:finalsum}
	    E(\Delta F_{k }) \leq  \sum_{i=1}^{k-1}  \frac{\max\{ba, \frac{1}{c}\}}{k-i}(1-p)p^{i-1}\leq\frac{ \max\{ba, \frac{1}{c}\}(1-p)}{k-1} 
	    \sum_{i=1}^{k-1} ( p^{i-1}+\frac{i-1}{k-i}p^{i-1}).
	\end{align}
	To bound the term $\sum_{i=1}^{k-1} ( p^{i-1}+\frac{i-1}{k-i}p^{i-1})$ observe that $\frac{i-1}{k-i}p^{i-1}\leq (i-1)p^{i-1}$ and hence
	\[
	\sum_{i=1}^{k-1} ( p^{i-1}+\frac{i-1}{k-i}p^{i-1})\leq \frac{1}{(1-p)}+\frac{1}{(1-p)^2}\leq \frac{2-p}{(1-p)^2},
	\]
which gives us the final bound on the expected error. 
\end{proof}

We note that the (2-p) factor  is the result of an overestimate of the weighted geometric series and a tighter bound should be possible to obtain.


Below we show that randomized coordinate descent can guarantee sufficient accuracy for subproblem solutions and hence maintain the sub linear convergence rate (in expectation) of Algorithm \ref{alg:ISTA-SD}. Moreover, we show in Section \ref{sec:comp}, that the randomized coordinate descent is as efficient in practice as the cyclic one.

\subsection{Analysis of Subproblem Optimization via Randomized Coordinate Descent} 
\label{sec:coordinate_descent_iteration_complexity}
In randomized coordinate descent the model function $Q(\cdot)$ is iteratively minimized over one randomly chosen coordinate, while the others remain fixed. 
The method  is presented in Algorithm \ref{alg:randomized_cd} and is applied for $l$ steps, with $l$ being an input parameter. 

\begin{algorithm2e}\caption{Randomized Coordinate Descent for optimizing Model Function $Q(H, v, x)$ over $v$: {\em RCD\ }$(Q(H, v, x), x, l)$ }
    \label{alg:randomized_cd}%
Set $p(x) \gets x $\; 
\For{$i=1,2,\cdots, l$}{
Choose $j \in \{ 1,2, ..., n\}$ with probability $\frac{1}{n}$\;
$z^* = \arg \underset{z}{\min}~ Q(H, p(x) + ze_j, x)$\;
$p(x) \gets p(x) + z^*e_j$\;
}
Return $p(x)$.
\end{algorithm2e}


Here we will show how properties of coordinate descent can be used together with the analysis in Section \ref{sec:complete_conv}. Combination of
coordinate descent with the waker analysis presented in Section \ref{sec:inexact_prox} can be found in \cite{OML}. 

Our analysis is based on Richtarik and Takac's results on iteration complexity of randomized coordinate descent  \cite{Richtarik2012}. In particular, we make use of Theorem 7 in \cite{Richtarik2012}, which we restate below without proof, while adapting it to our context.

\begin{lemma}\label{lem:randomized_CD}

Let $v$ be the initial point and $Q^* := \min_{u\in\br^n} Q(H, u,v)$. If $v_l$ is the random point generated by applying $l$ randomized coordinate descent steps to a strongly convex function $Q$, then for some constant 
we have
\begin{align*}
    P\{Q(H,v_l,v) - Q^*\geq \phi\} \leq p, 
\end{align*}
as long as 
\begin{align*}
i\geq n(1+\mu(H)) \log(\frac{Q(H,v,v) - Q^*}{\phi p}),
\end{align*}
where $\mu(H)$ is a constant that measures conditioning of $H$ along the coordinate directions and in the worst case is at most $M/\sigma$ - the condition 
number of $H$.
\end{lemma}

Let us now state the version of  Algorithms \ref{alg:ISTA-SD} and \ref{alg:Backtrack_SD} which is close to what implement in practice and 
discuss in the following sections  and for which the complexity  bound can be applied.

\begin{algorithm2e}\caption{Proximal Quasi-Newton method using randomized coordinate descent}
    \label{alg:ISTA-SD_RCD}%
{\rm Choose }
$0<\rho\leq 1$,  $a,b>0$ and  $x^0$\; 
\For{$k=0,1,2,\cdots$}{
 Choose $0<\bar \mu_k,  \theta_k>0, G_k\succeq 0$.\;
Find  $H_k=G_k+\frac{1}{2\mu_k}I$ and  $x^{k+1}  := p_{H_k,\phi_k}(x^k)$  \\ by applying {\em Prox\ Parameter\ Update\ with\ RCD\ }$(\bar \mu_k, G_k, x^k, \rho, a, b)$.\;
}
\end{algorithm2e}

\begin{algorithm2e}\caption{Prox Parameter Update with RCD $(\bar \mu, G, x, \rho, a,b)$ }
    \label{alg:Backtrack_SD_RCD}%
Select $0<\beta<1$ and set $\mu=\bar \mu$\; \For{$i=1,2,\cdots$}{
Define $H=G+\frac{1}{2\mu} I$,  and compute $p(x):=p_{H,\phi}(x)$\\ by applying {\em RCD \ }$(Q(H,  v,x), x, \lceil ak+b\rceil)$\;
If $F(p(x))- F(x) \leq \rho (Q(H,  p(x),x)- F(x))$, then output $H$ and $p(x)$, {\bf Exit} \;
Else $\mu=\beta^{i}\bar  \mu$\;
}
\end{algorithm2e}

 The  conclusion of Lemma \ref{lem:randomized_CD} in application to Algorithms \ref{alg:randomized_cd}-\ref{alg:Backtrack_SD_RCD} 
 is that  when applying $l(k)$ steps of  randomized coordinate descent 
 approximately to optimize $Q(H_k,  u,x^k)$, $\phi_k\leq M_Q\delta^{l(k)}$, with probability $p$, where  
 $M_Q$ is an upper bound on  $\frac{Q(H_k,x^k, x^k) - Q(H_k, p_H(x^k), x^k)}{p}$ and $\delta=e^{-\frac{1}{n(1+\mu(H_k))}}$.
 This, together with Theorem \ref{th:inexact_random_conv_rate} implies that to solve subproblem on iteration $k$ it is sufficient to set 
 $l(k)=O(n(1+\mu(H))\log(kp/M_Q))$ and the resulting convergence rate will then obey
  Theorems \ref{th:inexact_conv_rate} and \ref{th:inexact_total_conv_rate_klogk}. 
   However, it is necessary to know constants $M_Q$ 
   and $\mu(H)$ to be able to construct efficient expression $l(k)$.
 In practice, a successful strategy
 is to select a slow growing linear function  of $k$, $l(k)=ak+b$. This certainly guarantees convergence rate of the outer iteration as in Theorem \ref{th:inexact_random_conv_rate}. In terms of overall rate this gives inferior complexity, however, we believe that the real difference in
  terms of the workload appears only in the limit, while in most cases the algorithm successfully terminated before  our practical formula $l(k)=ak+b$, described in the next section, significantly exceeds, the theoretical bound $O(n(1+\mu(H))log(kp/M_Q))$ with appropriate constants. Moreover, as noted earlier, the number of function 
  and gradient evaluations may be the dominant complexity in the big data cases, 
  hence it may be worthwhile to increase workload of coordinate descent  in order to reduce the constants in the bound in
   Theorem \ref{th:inexact_random_conv_rate}. 
   
   Finally, we note that when using ISTA method for the subproblem, instead of randomized coordinate descent, 
   the number of inner iteration does not need to depend on the dimension $n$. However, 
   it depends on $\sigma_k/M_k$ and the cost per iteration is roughly $n$ times bigger than that of coordinate descent (with LBFGS matrices). 
   Hence the overall complexity of using coordinate descent is better than that of ISTA if $\mu(H_k)\ll \sigma_k/M_k$. 
   This indeed happens often in practical problems, as is discussed in \cite{Richtarik2012} and other works on coordinate descent.

\section{Optimization Algorithm}
\label{sec:alg}

In this section we briefly describe the specifics of the general purpose algorithm that we propose within the framework of Algorithms 
\ref{alg:ISTA-SD_RCD}, \ref{alg:Backtrack_SD_RCD} and \ref{alg:randomized_cd} and that takes advantage of approximate second order information while maintaining low complexity of subproblem optimization steps.
The algorithm is designed to solve problems of the form \eqref{prob:P} with $g(x)=\lambda \|x\|_1$, but it does not use any special structure of the smooth part of the objective, $f(x)$.

 At iteration $k$ a step $d_k$  is obtained, approximately, as follows
\begin{equation}
    \label{equ:compute_d_outer_Ak}
    \nonumber d_k = \arg\min_{d} \{\nabla f(x^k)^T d + d^T H_k d + \lambda\|x^k + d\|_1; \ \mbox{s.t.}~ d_i = 0, \forall i \in \Acal_k\},
\end{equation}
with $H_k=G_k+\frac{1}{2\mu_k}I$ - a positive definite matrix  and $\Acal_k$ -  a set of coordinates fixed at the current iteration.

The positive definite matrix $G_k$ is computed by a limited memory BFGS approach. In particular,
 we use a  specific form of Hessian estimate, (see e.g. \cite{Byrd1994,NoceWrig06}),
\begin{equation}
    \label{equ:LBFGS}
    G_k = \gamma_k I - QRQ^T = \gamma_k I - Q\hat Q \quad \mbox{with  } \hat Q= RQ^T,
\end{equation}
where $Q$, $\gamma_k$ and $R$ are defined below,
\begin{align}
    Q = 
    \begin{bmatrix}
        \gamma_kS_k &T_k
    \end{bmatrix}, ~
    R = 
    \begin{bmatrix}
        \gamma_kS_k^TS_k &M_k\\
        M_k^T   &-D_k
    \end{bmatrix}^{-1}, ~
    \gamma_k = \frac{t^T_{k-1}t_{k-1}}{t^T_{k-1}s_{k-1}}.   .
\end{align}
Note that there is low-rank structure present in $G_k$, the matrix given by $Q \hat Q$, which we can exploit, but $G_k$ itself by definition is always positive definite. 
Let $m$ be a small integer which defines the number of latest
BFGS updates that are "remembered" at any given iteration (we used $10-20$).
Then $S_k$ and $T_k$ are the $p \times m$ matrices with columns defined by  vector pairs $\{s_i,t_i\}_{i=k-m}^{k-1}$ that satisfy $s_i^Tt_i > 0, s_i = x^{i+1} - x_i$ and $t_i = \nabla f(x^{i+1}) - \nabla f(x^{i}) $, 
 $M_k$ and $D_k$ are the $k \times k$ matrices
    \begin{equation*}
        (M_k)_{i,j} = 
        \begin{cases}
            s_{i-1}^Tt_{j-1} \quad &\mbox{if }i > j \\
            0   \quad &\mbox{otherwise,}
        \end{cases} \quad D_k = \diag[s^T_{k-m}t_{k-m},...,s_{k-1}^Tt_{k-1}].
    \end{equation*}
 The particular choice of $\gamma_k$ is meant  to promote   well-scaled quasi-Newton steps, so that less time is spent on line search or updating of prox parameter $\mu_k$ \cite{NoceWrig06}. In fact instead of updating and maintaining $\mu_k$, exactly as described in Algorithm \ref{alg:Backtrack_SD}  we simply
 double $\gamma_k$ in \eqref{equ:LBFGS} at each backtracking step. This can be viewed as choosing   $\mu_k=\infty$ for the first step of backtracking, and $\mu_k=1/(2^{i-1}-1)\gamma_k$ for the $i$-th backtracking step, when $i>1$. As long as $G_K$ in \eqref{equ:LBFGS}, is positive definite, with smallest eigenvalue bounded by $\sigma>0$, our theory applies to this particular backtracking procedure. 


\subsection{Greedy Active-set Selection $\mathcal{A}_k(\Ical_k)$} 
\label{ssub:greedy_active_set_selection}
An active-set selection strategy maintains a sequence of sets of indices $\Acal_k$ that iteratively estimates the optimal active set $\Acal^*$ which contains indices of zero entries in the optimal solution $x^*$ of (\ref{prob:P}). We introduce this strategy as a heuristic aiming to improve the efficiency of the implementation and to make it comparable with state-of-the-art methods, which also use active set strategies. A theoretical analysis of the effects of these strategies is a subject of future study. 
The complement set of $\Acal_k$ is $\Ical_k = \{i \in \Pcal ~|~ i \notin \Acal_k\}$. 
Let $(\partial F(x^k))_i$ be the $i$-th component of a subgradient of $F(x)$ at $x^k$. We define two sets, 
\begin{align}
    \Ical^{(1)}_k &= \{i \in \Pcal ~|~(\partial F(x^k))_i \neq 0\}, \ \Ical^{(2)}_k = \{i \in \Pcal~|~ (x^k)_i \neq 0\}.
\end{align}
As is done in  \cite{nGLMNET} and \cite{Hsieh2011} we select $\Ical_k$ to include the entire set $\Ical^{(2)}_k$ and the entire set $\Ical^{(1)}_k$. We also  tested a strategy which  includes only  a small subset of  indices from $\Ical^{(1)}_k$ for which the corresponding elements $|(\partial F(x^k))_i|$    are the largest.  This strategy  resulted  in a smaller size of subproblems (\ref{equ:compute_d_outer_Ak}) at the early stages of the algorithm, but did not appear to improve the overall performance of the algorithm.

\subsection{Solving the inner problem via coordinate descent}\label{sec:innerprob} 
\label{sub:the_algorithm}
We apply coordinate descent method to the  piecewise quadratic subproblem (\ref{equ:compute_d_outer_Ak}) to obtain the direction $d_k$ and exploit the special structure of $H_k$. 
Suppose $j$-th coordinate in $d$ is updated, hence $d' = d + ze_j$ ($e_j$ is the $j$-th vector of the identity). Then $z$ is obtained by solving the following one-dimensional problem
\begin{align}
    \label{equ:compute_z_prob}
    \nonumber\min_z  &(H_k)_{jj} z^2 + ( (\nabla f(x^k))_j + (2H_kd)_j  )z + \lambda | (x^k)_j + d_j + z |,
\end{align}
which has a simple closed-form solution \cite{Donoho92de-noisingby,Hsieh2011}. 

The most costly step of an iteration of the coordinate descent method is computing or maintaining vector  $H_kd$. Naively, or in the case of general $H_k$, this step takes $O(n)$ flops, since the vector needs to be updated at the end of each iteration, when one of the coordinates of vector $d$ changes.  The special form of  $G_k$ in $H_k=G_k+\frac{1}{\mu} I=\gamma_kI-Q\hat Q$ gives us an opportunity to accelerate this step, reducing the complexity from problem-dependent $O(n)$ to $O(m)$ with $m$ chosen as a small constant. In particular we only store the diagonal elements of $G_k$, $(G_k)_{ii} = \gamma_k - q_i^T\hat q_i$, where $q_i$ is the $i$th row of the matrix $Q$ and $\hat q_i$ is the $i$th column vector of the matrix $\hat Q$. We compute $(G_kd)_i$,  whenever it is needed, by 
 maintaining a $2m$ dimensional vector $v := \hat Qd$, 
which takes $O(2m)$ flops, and using  $(G_kd)_i = \gamma_k d_i - q_i^T v$.
After each coordinate step $v$ is updated by $v \gets v + z_i \hat q_i$, which costs $O(m)$. We also need to use extra memory for caching $\hat Q$ and $\hat d$ which takes $O(2mp + 2m)$ space. With the other $O(2p + 2mn)$ space for storing the diagonal of $G_k$, $Q$ and $d$, altogether we need $O(4mp + 2n + 2m)$ space, which is essentially $O(4mn)$ when $n \gg m$.

\section{Computational experiments} 
\label{sec:comp}

The aim of this section is to provide validation for our general purpose algorithm, but not to conduct extensive comparison of various inexact proximal Newton approaches. 
In particular, we aim to demonstrate a) that using the exact Hessian is not necessary in these methods, b) that backtracking using prox parameter, based on sufficient decrease condition, which our theory uses, does in fact work well in practice and c) that randomized coordinate descent is at least as effective as the cyclic one, which is standardly used by other methods.

LHAC, for \textbf{L}ow rank \textbf{H}essian \textbf{A}pproximation in \textbf{A}ctive-set \textbf{C}oordinate descent, is a C/C++ package that implements Algorithms \ref{alg:randomized_cd}-\ref{alg:ISTA-SD_RCD} for solving general $\ell_1$ regularization problems. We conduct experiments on two of the most well-known $\ell_1$ regularized models -- Sparse Inverse Covariance Selection (SICS) and Sparse Logistic Regression (SLR). The following two specialized C/C++ solvers are included in our comparisons:

\begin{itemize}
    \item QUIC: the quadratic inverse covariance algorithm for solving SICS described in \cite{Hsieh2011}. 
    \item LIBLINEAR: an improved version of GLMNET for solving SLR described in \cite{GLMNET,nGLMNET}.
\end{itemize}

Note that both of these  packages have been shown to be the state-of-the-art solvers in their respective categories (see e.g. \cite{nGLMNET,Yuan2010,Hsieh2011,Olsen2012}). 

Both QUIC and LIBLINEAR adopt line search to ensure function reduction. We have  implemented line search in LHAC as well to see how it compares to the updating of prox parameter proposed in Algorithm \ref{alg:Backtrack_SD}.  In all the experiments presented below use the following notation. 
\begin{itemize}
    \item LHAC: Algorithm \ref{alg:ISTA-SD_RCD} with backtracking on prox parameter.
    \item LHAC-L: Algorithm \ref{alg:ISTA-SD_RCD} with Armijo line search procedure described below in \eqref{equ:line_search}.
\end{itemize}


\subsection{Experimental Settings} 
\label{sub:data_sets_and_experimental_settings}

For all of the experiments we choose the initial point $x_0 = \mathbf{0}$, and we report running time results in seconds, plotted against log-scale relative objective function decrease given by
\begin{align}        
    \log(\frac{ F(x) - F^* }{ F^* }),
\end{align}
where $F^*$ is the optimal function value. Since $F^*$ is not available, we compute an approximation by setting a small optimality tolerance,  specifically $10^{-7}$, in QUIC and LIBLINEAR. All the experiments are executed through the MATLAB mex interface. 
We also modify the source code of LIBLINEAR in both its optimization routine and mex gateway function to obtain the records of function values and the running time. We note that we simply store, in a double array, and pass the function values which the algorithm already computes,  so this adds little to nothing to LIBLINEAR's computational costs. We also adds a function call of \emph{clock()} at every iteration to all the tested algorithms, except  QUIC, which includes  a ``trace'' mode that returns automatically the track of function values and running time,  by calling \emph{clock()} iteratively. For both QUIC and LIBLINEAR we downloaded the latest versions of the publicly available source code from their official websites, compiled and built the software on the  machine on which all experiments were executed, and which uses 2.4GHz quad-core Intel Core i7 processor, 16G RAM and Mac OS.  

The optimal objective values $F^*$ obtained approximately by QUIC and LIBLINEAR are later plugged in LHAC and LHAC-L to terminate the algorithm when the following condition is satisfied
\begin{align}
    \frac{ F(x) - F^* }{ F^* } \leq 10^{-8}.
\end{align}

In LHAC we chose $\bar \mu = 1, \beta = 1/2$ and $\rho = 0.01$ for sufficient decrease (see Algorithm \ref{alg:Backtrack_SD_RCD}), and for LBFGS we use $m = 10$. 

When solving the subproblems, we terminate the RCD procedure whenever the number of coordinate steps exceeds
\begin{align}        
\label{cond:subprob_tol}
     (1 + \lfloor \frac{k}{m} \rfloor ) |\Ical_k|,
 \end{align} 
where $|\Ical_k|$ denotes the number of coordinates in the current working set. Condition \eqref{cond:subprob_tol} indicates that we expect to update each coordinate in $\Ical_k$ only once when $k < m$, and that when $k > m$ we increase the number of expected passes through $\Ical_l$ by 1 every $m$ iterations, i.e., after LBFGS receives a full update. The idea is not only to avoid spending too much time on the subproblem especially at the beginning of the algorithm when the Hessian approximations computed by LBFGS are often fairly coarse, but also to solve the subproblem more accurately as the iterate moves closer to the optimality. Note that in practice when $|\Ical_k|$ is large, the value of  \eqref{cond:subprob_tol} almost always dominates $k$, hence it can be lower bounded by $l(k)=ak+b$ with some reasonably large values of $a$ and $b$, which, as we analyzed in Section \ref{sec:coordinate_descent_iteration_complexity},  guarantees the sub linear convergence rate. We also find that \eqref{cond:subprob_tol} works quite well in practice in preventing from ``over-solving'' the subproblems, particularly for LBFGS type algorithms. In Figures \ref{fig:sics-cd} we plot the data with respect to the number of RCD iterations. In particular Figures \ref{fig:sics-cd-iter} and \ref{fig:slr-cd-iter} show the number of RCD steps taken at the $k$-th iteration, as a function of  $k$. Figures \ref{fig:sics-cd-obj} and \ref{fig:slr-cd-obj} show convergence of the objective function to its optimal value as a function of the total number of RCD steps taken so far (both values are plotted in logarithmic scale). Note that RCD steps are not the only component of the CPU time of the algorithms, since gradient computation has to be performed at least once per iteration. 

In LHAC-L,  a line search procedure is employed, as is done  in QUIC and LIBLINEAR,   for the convergence  to follow from the framework by \cite{Tseng2009}. In particular, the Armijo rule chooses the step size $\alpha_k$ to be the largest element from $\{\beta^0, \beta^1, \beta^2, ... \}$ satisfying
\begin{align}
    \label{equ:line_search}
    F(x_k + \alpha_k d_k) \leq F(x_k) + \alpha_k \sigma \Delta_k,
\end{align} 
where $0 < \beta < 1, 0 < \sigma < 1$, and $\Delta_k := \nabla f_k^T d_k + \lambda \|x_k + d_k\|_1 - \lambda \|x_k\|_1$. In all the experiments we chose $\beta = 0.5, \sigma = 0.001$ for LHAC-L.


\subsection{Sparse Inverse Covariance Selection} 
\label{sub:sparse_inverse_covariance_matrix_estimation}

\begin{figure}[!t]
    \centering
    \subfigure[$S: 1255 \times 1255 $]
    {
       \includegraphics[width=0.2\columnwidth]{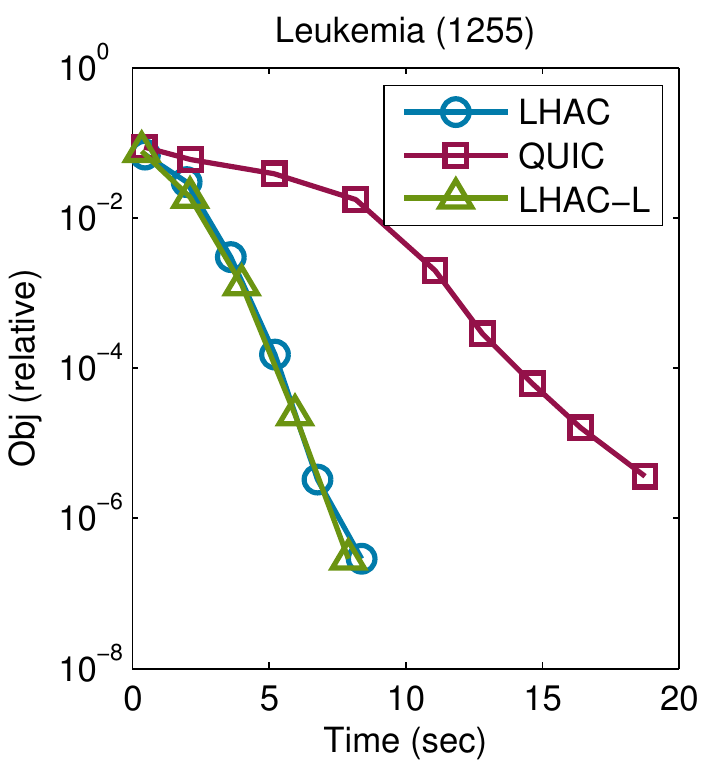}
       \label{fig:Leukemia}
    }
    ~
    \subfigure[ $S: 1869 \times 1869$ ]
    {
       \includegraphics[width=0.2\columnwidth]{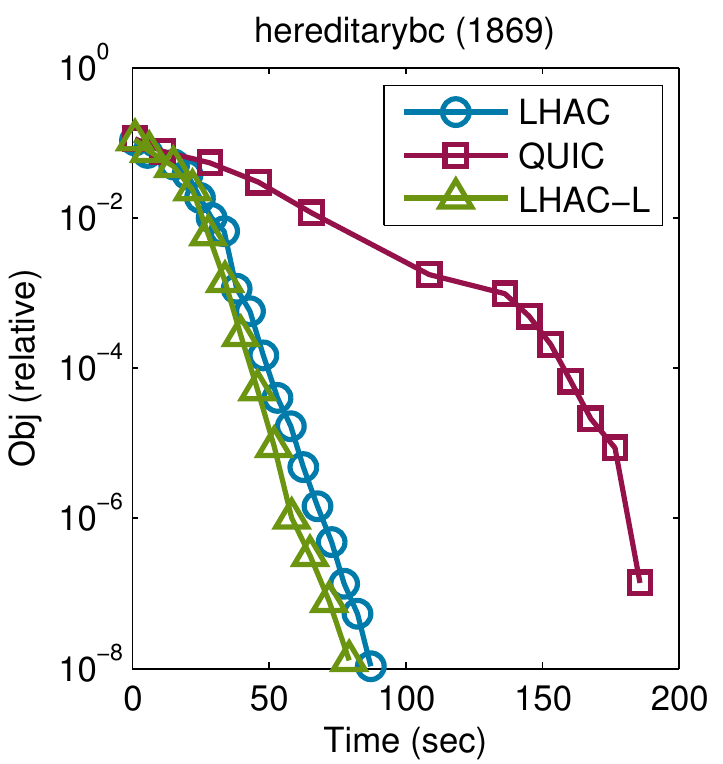}
       \label{fig:hereditarybc}
    }
    ~
    \subfigure[ $S:  692 \times 692$ ]
    {
       \includegraphics[width=0.2\columnwidth]{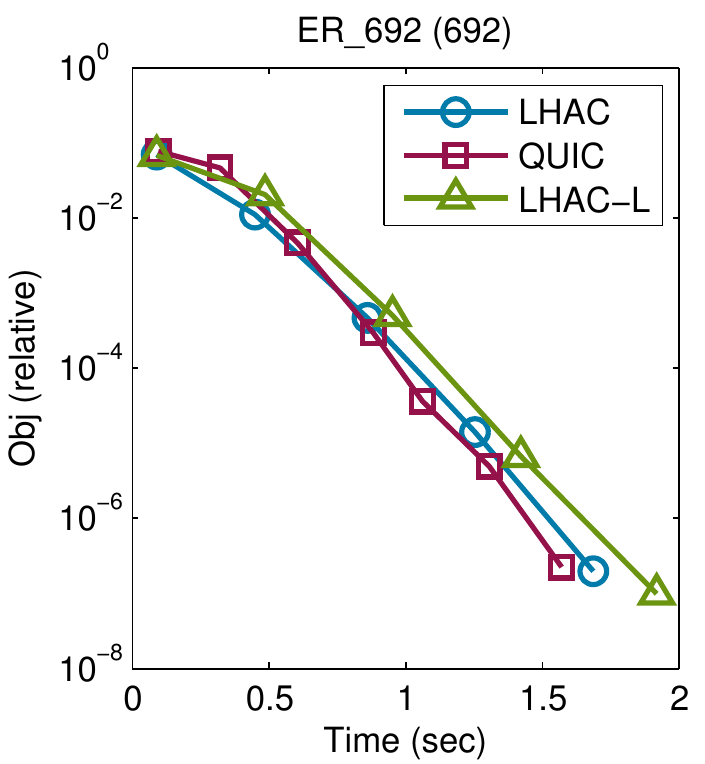}
       \label{fig:Leukemia}
    }
    ~
    \subfigure[ $S: 834 \times 834$ ]
    {
       \includegraphics[width=0.2\columnwidth]{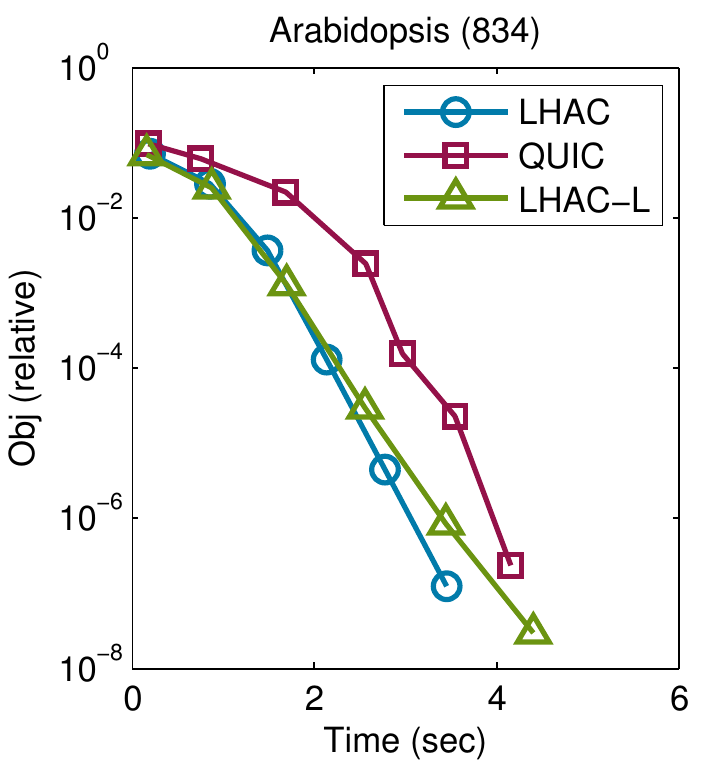}
       \label{fig:hereditarybc}
    }
    \caption{ Convergence plots on SICS (the y-axes on log scale). }
    \label{fig:sics}    
\end{figure}

\begin{figure}[!t]
    \centering
    \subfigure[The number of coordinate descent steps on subproblems w.r.t. outer iteration.]
    {
       \includegraphics[width=0.2\columnwidth]{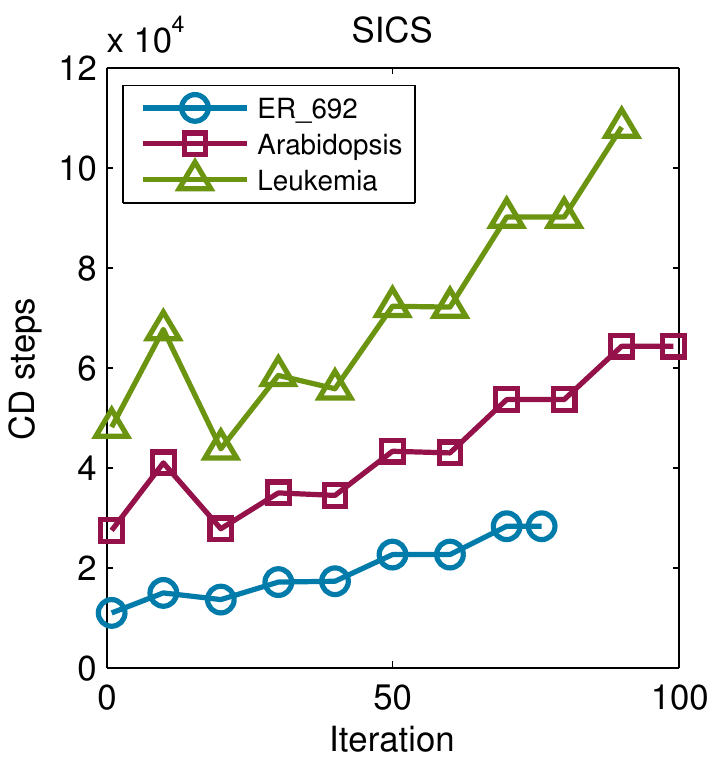}
       \label{fig:sics-cd-iter}
    }
    ~
    \subfigure[ Both axes are in log scale. Change of objective w.r.t. the number of coordinate descent steps.]
    {
       \includegraphics[width=0.2\columnwidth]{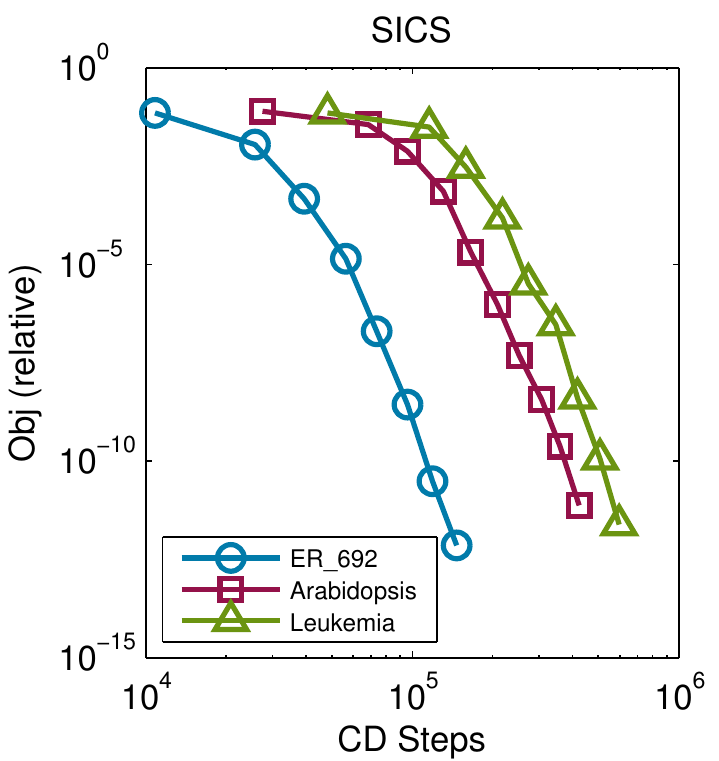}
       \label{fig:sics-cd-obj}
    }
    ~ 
    \subfigure[The number of coordinate descent steps on subproblems w.r.t. outer iteration.]
    {
       \includegraphics[width=0.2\columnwidth]{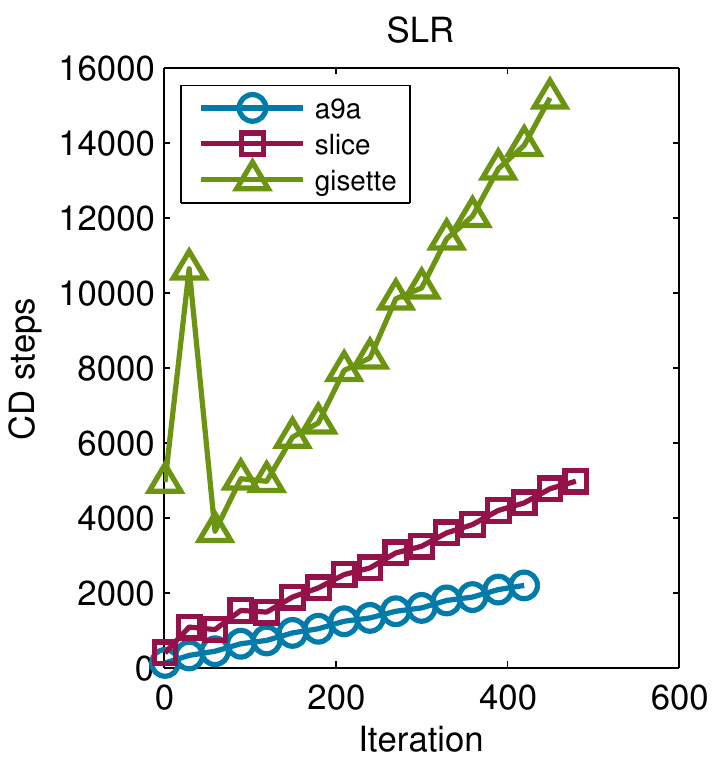}
       \label{fig:slr-cd-iter}
    }
    ~
    \subfigure[ Both axes are in log scale. Change of objective w.r.t. the number of coordinate descent steps.]
    {
       \includegraphics[width=0.2\columnwidth]{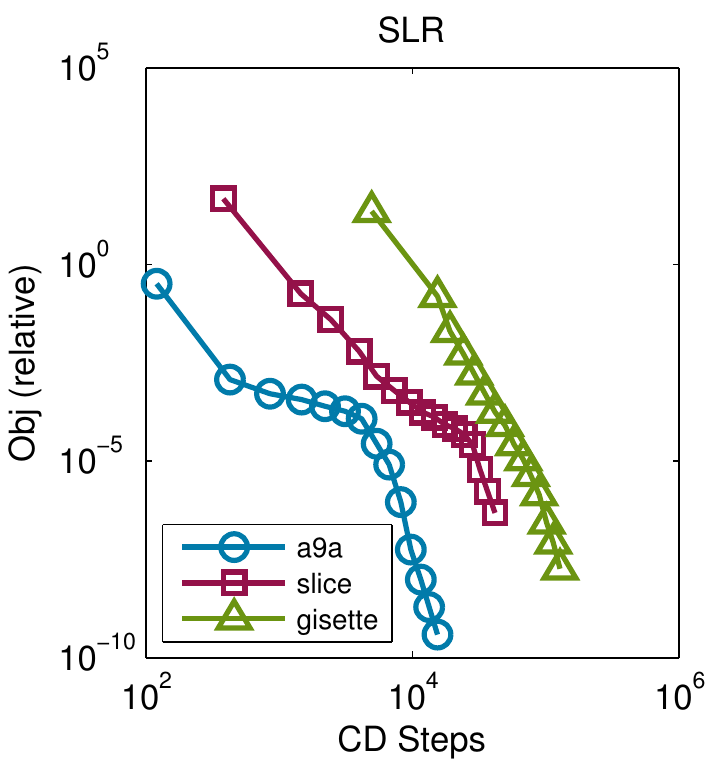}
       \label{fig:slr-cd-obj}
    } 
    \caption{ RCD step count of LHAC on different SLR and SICS data sets. }
    \label{fig:sics-cd}   
\end{figure}

The sparse inverse covariance selection problem is defined by
\begin{align}
    \label{equ:sics_obj}
    \min_{X \succ 0} \quad F(X) = -\log \det X + \tr(SX) + \lambda ||X||_1,
\end{align}
where the input $S \in \Rmbb^{p \times p}$ is the sample covariance matrix and the optimization is over a symmetric matrix  $X \in \Rmbb^{p \times p}$ that is required to be positive definite.

For SICS we report results on four real world data sets, denoted as \emph{ER\_692}, \emph{Arabidopsis}, \emph{Leukemia} and \emph{hereditarybc}, which are preprocessed from breast cancer data and gene expression networks. We refer to \cite{Li2010} for detailed information about those data sets.

We set the regularization parameter $\lambda = 0.5$ for all experiments as suggested in \cite{Li2010}. 
The plots presented in Figure \ref{fig:sics} show that LHAC and LHAC-L is almost twice as fast as QUIC,  in the two largest data sets Leukemia and hereditarybc (see Figure \ref{fig:Leukemia} and \ref{fig:hereditarybc}). In the other two smaller data sets the results are less clear-cut, but all of the methods solve the problems very fast and the performance of LHAC is comparable to that of QUIC. The performances of LHAC and LHAC-L are fairly similar in all experiments. Again we should note that with the sufficient decrease condition proposed in Algorithm \ref{alg:Backtrack_SD} we are able to establish the global convergence rate, which has not been shown  in the case of Armijo line search.

\subsection{Sparse Logistic Regression} 
\label{sub:sparse_logistic_regression}
\begin{figure}[!t]
    \centering
    \subfigure[ a9a ($123, 32561$)]
    {
       \includegraphics[width=0.2\columnwidth]{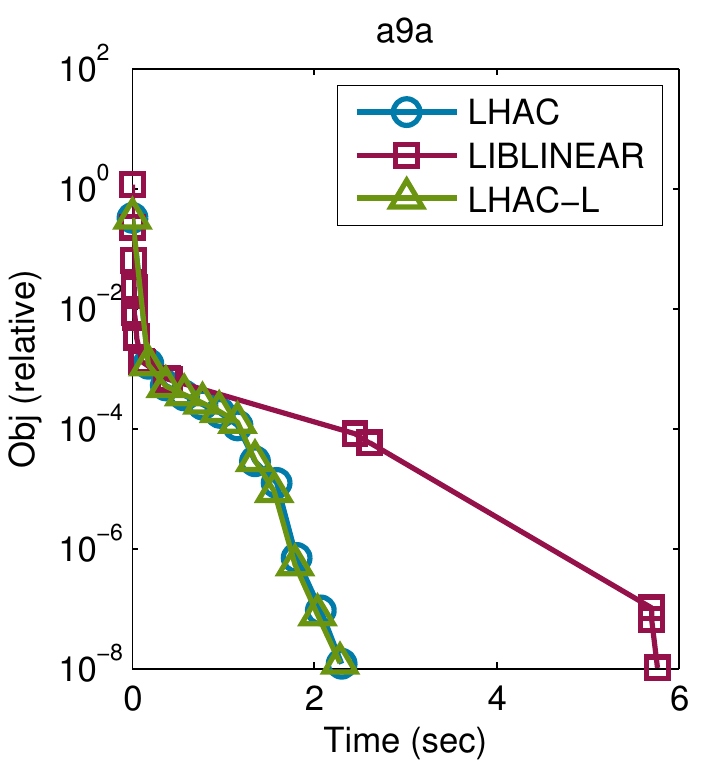}
       \label{fig:a9a}
    }
    ~
    \subfigure[ slices ($385, 53500$)]
    {
       \includegraphics[width=0.2\columnwidth]{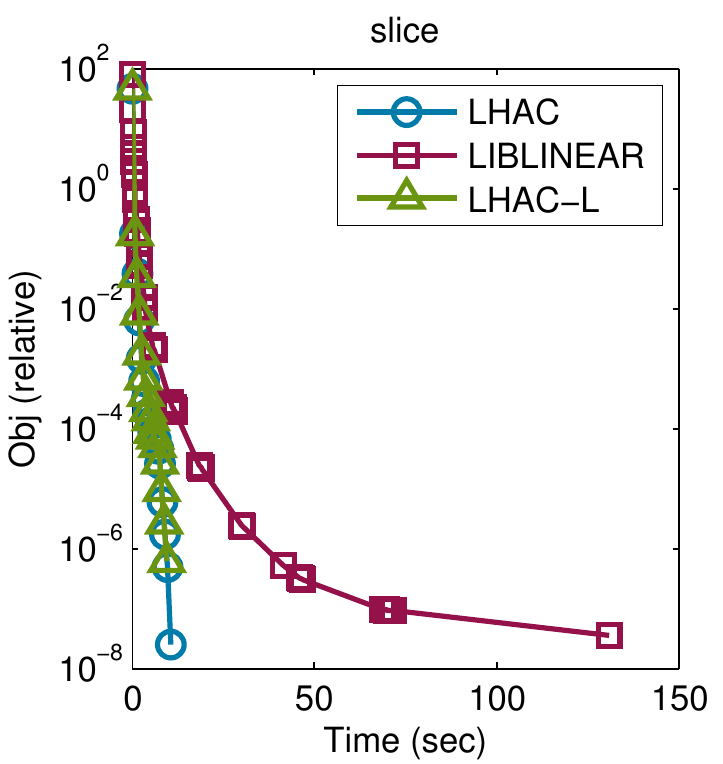}
       \label{fig:slices}
    }
    ~    
    \subfigure[ gisette ($5000, 6000$)]
    {
       \includegraphics[width=0.2\columnwidth]{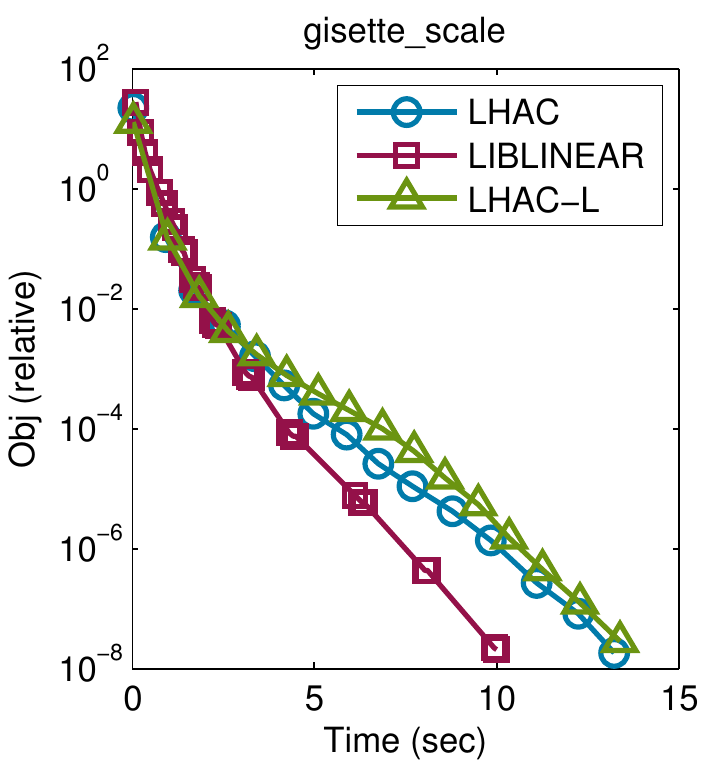}
       \label{fig:gisette}
    }
    ~
    \subfigure[epsilon ($2000, 1e^5$)]
    {
       \includegraphics[width=0.2\columnwidth]{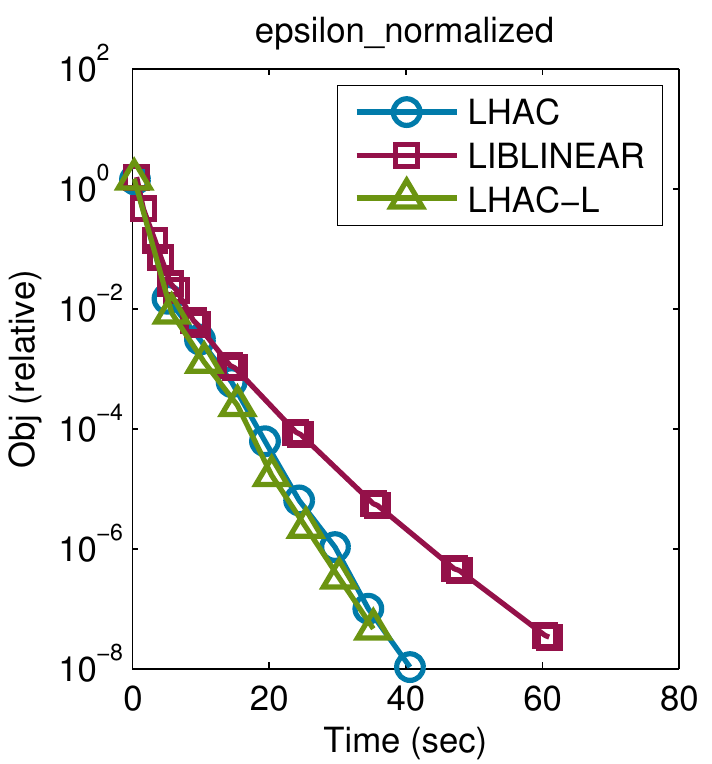}
       \label{fig:connect}
    }
    \caption{ Convergence plots on SLR. The y-axes is on log scale. Two numbers in parenthesis are features $p$ and sample size $N$. }
    \label{fig:log_reg2}
\end{figure}


The objective function of sparse logistic regression is given by
\begin{align*}
    F(w) = \lambda \|w\|_1 + \frac{1}{N} \sum_{n=1}^N \log(1 + \exp(-y_n \cdot w^Tx_n)),
\end{align*}
where $L(w) = \frac{1}{N} \sum_{n=1}^N \log(1 + \exp(-y_n \cdot w^Tx_n))$ is the
average logistic loss function and $ \{ ( x_n, y_n ) \}^N_{n=1} \in ( \Rmbb^p \times \{-1,1\} ) $ is the training set. The number of instances in the training set and the number of features are denoted by $N$ and $p$ respectively. Note that the evaluation of $F$ requires $O(pN)$ flops and  to compute the Hessian requires $O(Np^2)$ flops. Hence, we chose such training sets for our experiment with $N$ and $p$  large enough to test the scalability of the algorithms and yet small enough to be completed on a workstation. 

We report results of SLR on four data sets downloaded from UCI Machine Learning repository \cite{Bache+Lichman:2013}, whose statistics are summarized in Table~\ref{tab:Data_statistics}. In particular, the first data set is the well-known UCI Adult benchmark set \emph{a9a} used for income classification, determining whether a person makes over \$50K/yr or not, based on census data; the second one we use in the experiments is called \emph{epsilon}, an artificial data set for PASCAL large scale learning challenge in 2008; 
the third one, \emph{slices}, contains features extracted from CT images and is often used for predicting the relative location of CT slices on the human body; and finally we consider \emph{gisette}, a handwritten digit recognition problem from NIPS 2003 feature selection challenge, with the feature set of size 5000 constructed in order to discriminate between two confusable handwritten digits: the four and the nine. 

\begin{table}[h!]
\begin{center}
\begin{tabular}{|l|r|r|r|p{5cm}|}
    \hline
              \textbf{Data set}    & \#features $p$ & \#instances $N$ & \#non-zeros & Description \\
    \hline
        \textbf{a9a} &  $123$ & $32561$ & $451592$   & `census Income' dataset. \\
    \hline        
        \textbf{epsilon} &  $2000$ & $100000$ & $200000000$   &  PASCAL challenge 2008. \\
    \hline    
        \textbf{gisette} &  $5000$ & $6000$ & $29729997$  & handwritten digit recognition. \\
    \hline        
        \textbf{slices} &  $385$ & $53500$ & $20597500$   &  CT slices location prediction.\\   
    \hline
\end{tabular}
\end{center}
\caption{ Data statistics in sparse logistic regression experiments. }
\label{tab:Data_statistics}
\end{table}

The results are shown in Figure~\ref{fig:log_reg2}. In most cases LHAC and LHAC-L outperform LIBLINEAR. 
On data set \emph{slice}, LIBLINEAR experiences difficulty in convergence which results in LHAC being faster by an order of magnitude.  
On the largest data set \emph{epsilon}, LHAC and LHAC-L is faster than LIBLINEAR by about one third and reaches the same precision. Finally we note that the memory usage of LIBLINEAR is more than doubled compared to that of LHAC and LHAC-L, as we observed in all the experiments and is particularly notable on  the largest data set \emph{epsilon}.

\section{Conclusion} In this paper we presented analysis of global convergence rate of inexact proximal quasi-Newton framework, and showed that randomized coordinate descent, among other subproblem methods,  can be used effectively to find inexact quasi-Newton directions, which guarantee sub linear convergence rate of the algorithm, in expectation. This is the first global convergence rate result for an algorithm that uses coordinate descent to inexactly optimize subproblems at each iteration. Moreover, we improve upon results for inexact proximal gradient method in \cite{Schmidtetal} in that our requirements on the error in subproblem solution are 
weaker and the resulting bound on the total number of inner solver iterations is smaller. 

Our framework does not rely on or exploit the accuracy of second order information, and hence we do not obtain fast local convergence rates. We also do not assume strong convexity of our objective function, hence a sublinear conference rate is  the best global rate we can hope to obtain. In \cite{Jiangetal2012} an accelerated scheme related to our framework is studied and an optimal sublinear convergence rate is shown, but  the assumptions on the Hessian approximations are a lot stronger in  \cite{Jiangetal2012} than in our paper, hence the accelerated method is not as widely applicable. The framework studied by us in this paper covers several existing efficient algorithms for large scale sparse optimization. However, to provide convergence rates we had to depart from some standard techniques, such as line-search, replacing it instead by a prox-parameter updating mechanism with a trust-region-like sufficient decrease condition for acceptance of iterates. We also use randomized coordinate descent instead of a cyclic one. We demonstrated  that this modified framework is,  nevertheless, very effective in practice and is competitive with state-of-the-art specialized methods.

\bibliographystyle{siam}
\bibliography{All,New_bib}

\end{document}